\documentclass{article}

\usepackage{arxiv}

\usepackage[utf8]{inputenc}
\usepackage[T1]{fontenc}
\usepackage{hyperref}
\usepackage{url}
\usepackage{booktabs}
\usepackage{amsfonts}
\usepackage{nicefrac}
\usepackage{microtype}
\usepackage{graphicx}

\usepackage[natbib=true,
style=ieee,
citestyle=numeric-comp,
maxcitenames=1, mincitenames=1,
maxbibnames=999,
url=false,isbn=false,doi=false]{biblatex}
\DefineBibliographyStrings{english}{
  andothers = {et al\adddot}
}
\AtEveryBibitem{
  \clearlist{language}
  \clearfield{series}
  \ifentrytype{book}{}{
      \clearfield{issn} 
      \clearfield{doi} 
      \clearlist{address}
      \clearfield{address}
      \clearfield{month}
      \clearname{editor}
      \clearfield{eprint}
      \clearlist{location}
      \clearfield{chapter}
  }
  \ifentrytype{online}{}{
    \clearfield{url}
  }
  \ifentrytype{inproceedings}{
    \clearlist{publisher}
  }{}
}
\newcommand{\ArchCondText}{Architecture condition}

\usepackage{amsmath,amsfonts,bm}
\usepackage{amsthm}
\newtheorem{theorem}{Theorem}
\newtheorem{corollary}{Corollary}
\newtheorem{proposition}{Proposition}
\newtheorem{lemma}{Lemma}
\newtheorem{fact}{Fact}

\theoremstyle{definition}
\newtheorem{definition}{Definition}
\theoremstyle{remark}
\newtheorem*{remark}{Remark}

\usepackage{tikz}
\usetikzlibrary{shapes.geometric}
\usetikzlibrary {shapes.misc}
\usetikzlibrary{positioning}
\usepackage{lscape}
\usepackage{fontawesome}

\usepackage{subcaption}

\usepackage{enumerate}

\usepackage{mleftright}

\usepackage[toc,page]{appendix}

\newcommand\blfootnote[1]{
  \begingroup
  \renewcommand\thefootnote{}\footnote{#1}
  \addtocounter{footnote}{-1}
  \endgroup
}

\author{
Takeshi Teshima\textsuperscript{*}\\
The University of Tokyo, RIKEN\\
\texttt{teshima@ms.k.u-tokyo.ac.jp}\\
\And
Isao Ishikawa\textsuperscript{*}\\
Ehime University, RIKEN\\
\texttt{ishikawa.isao.zx@ehime-u.ac.jp}\\
\And
Koichi Tojo\\
RIKEN\\
\texttt{koichi.tojo@riken.jp}\\
\AND
Kenta Oono\\
The University of Tokyo\\
\texttt{kenta\_oono@mist.i.u-tokyo.ac.jp}\\
\And
Masahiro Ikeda\\
RIKEN\\
\texttt{masahiro.ikeda@riken.jp}\\
\And
Masashi Sugiyama\\
RIKEN, The University of Tokyo\\
\texttt{sugi@k.u-tokyo.ac.jp}
}

\newcommand{\acknowledgmentContent}{
The authors would like to thank the anonymous reviewers for the insightful discussions. We would also like to thank Dr. Taiji Suzuki, Associate Professor of the University of Tokyo, for his valuable comments and fruitful discussions on the distributional universality.
This work was supported by RIKEN Junior Research Associate Program.
TT was supported by Masason Foundation.
II and MI were supported by CREST:JPMJCR1913.
MS was supported by KAKENHI 20H04206.
}

\hypersetup{
pdftitle={Coupling-based Invertible Neural Networks Are Universal Diffeomorphism Approximators},
pdfsubject={},
pdfauthor={Takeshi Teshima, Isao Ishikawa, Koichi Tojo, Kenta Oono, Masahiro Ikeda, Masashi Sugiyama},
pdfkeywords={Invertible neural network, Normalizing flow, Affine coupling, Diffeomorphism, Universal approximation property, Universality},
}

\date{}

\usepackage{xcolor}

\newcommand{\status}[1]{}

\newcommand{\Restrict}[2]{#1\vert_{#2}}

\def \R {\mathbb{R}}
\def \Re {\mathbb{R}}

\newcommand{\Jac}[1]{D#1}
\newcommand{\vol}[1]{{\rm vol}(#1)}

\newcommand{\Aff}[3]{\Psi_{#1,#2,#3}}

\newcommand{\Na}{\mathbb{N}}
\newcommand{\ReD}{\Re^d}
\newcommand{\ReDminus}{\Re^{d-1}}

\newcommand{\DcRD}{\mathrm{Diff}^2_\mathrm{c}}
\newcommand{\Identity}{\mathrm{Id}}

\newcommand{\AutoRegressive}[3]{\AutoRegressiveFn{}_{#1,#2,#3}}
\newcommand{\AutoRegressiveFn}{h}

\newcommand{\upto}[2]{{#2_{{}\leq #1}}}
\newcommand{\from}[2]{{#2_{{}> #1}}}
\newcommand{\x}{\mbox{\boldmath $x$}}
\newcommand{\ba}{\mbox{\boldmath $a$}}

\newcommand{\CF}{CF}
\newcommand{\CFs}{\CF{}s}

\newcommand{\ACF}{ACF}
\newcommand{\ACFs}{\ACF{}s}

\newcommand{\CtwoDomainDiff}{{\mathcal{D}^2}}
\newcommand{\cptDomainDiff}{{\mathrm{Diff}^2_\mathrm{c}}}
\newcommand{\OneDimTriangularCmd}[1]{\mathcal{S}^{#1}_\mathrm{c}}
\newcommand{\CinftyOneDimTriangular}{\OneDimTriangularCmd{\infty}}
\newcommand{\ConeOneDimTriangular}{\OneDimTriangularCmd{1}}
\newcommand{\CzeroOneDimTriangular}{\OneDimTriangularCmd{0}}
\newcommand{\OneDimTriangular}{\CinftyOneDimTriangular}
\newcommand{\CtwoOneDimTriangular}{\OneDimTriangularCmd{2}}
\newcommand{\CrOneDimTriangular}{\OneDimTriangularCmd{r}}
\newcommand{\Triangular}{\mathcal{T}^\infty}

\newcommand{\LpKnorm}[1]{\left\Vert #1 \right\Vert_{p, K}}
\newcommand{\supRangenorm}[2]{\left\Vert #2 \right\Vert_{\sup, #1}}
\newcommand{\LpRangenorm}[2]{\left\Vert #2 \right\Vert_{p, #1}}
\newcommand{\inftyKnorm}[1]{\supRangenorm{K}{#1}}
\newcommand{\supKnorm}[1]{\supRangenorm{K}{#1}}
\newcommand{\Euclideannorm}[1]{\left\Vert #1 \right\Vert}
\newcommand{\opnorm}[1]{\left\Vert #1 \right\Vert_{\mathrm{op}}}
\newcommand{\ARFINNModel}{\mathcal{M}}
\newcommand{\ARFINNFlow}{\mathcal{G}}
\newcommand{\FACF}{\mathrm{ACF}}

\newcommand{\FDSF}{\mathrm{DSF}}
\newcommand{\FSoS}{\mathrm{SoS}}

\newcommand{\FLin}{\mathrm{Aff}}
\newcommand{\FGL}{\mathrm{GL}}

\newcommand{\Cinfty}{C^\infty}

\newcommand{\CcinftyRDminus}{C^\infty_c(\ReDminus)}

\newcommand{\ACFINNUniversalClass}{\mathcal{H}}

\newcommand{\FSACFcmd}[1]{#1\text{-}\mathrm{ACF}}
\newcommand{\FSACFH}{\FSACFcmd{\ACFINNUniversalClass}}

\newcommand{\INN}[1]{\mathrm{INN}_{#1}}
\newcommand{\FACFINN}{\INN{\FACF}}

\newcommand{\ARFINN}{CF-INN}
\newcommand{\ARFINNs}{\ARFINN{}s}

\newcommand{\ACFINN}{\(\FACFINN\)}

\newcommand{\FACFHINN}{\INN{\FSACFH}}
\newcommand{\ACFHINN}{\(\FACFHINN\)}
\newcommand{\INNModelGeneric}{\mathcal{M}}
\newcommand{\Indicator}[1]{\mathbf{1}_{#1}}

\newcommand{\psinstar}{\psi_n^*}
\newcommand{\psimstar}{\psi_n^*}

\title{Coupling-based Invertible Neural Networks Are Universal Diffeomorphism Approximators}

\begin{document}

\maketitle
\blfootnote{\textsuperscript{*}Equal contribution.}

\begin{abstract}
Invertible neural networks based on coupling flows (CF-INNs) have various machine learning applications such as image synthesis and representation learning. However, their desirable characteristics such as analytic invertibility come at the cost of restricting the functional forms. This poses a question on their representation power: are CF-INNs \emph{universal approximators} for invertible functions? Without a universality, there could be a well-behaved invertible transformation that the CF-INN can never approximate, hence it would render the model class unreliable. We answer this question by showing a convenient criterion: a CF-INN is universal if its layers contain affine coupling and invertible linear functions as special cases. As its corollary, we can affirmatively resolve a previously unsolved problem: whether normalizing flow models based on affine coupling can be \emph{universal distributional approximators}. In the course of proving the universality, we prove a general theorem to show the equivalence of the universality for certain diffeomorphism classes, a theoretical insight that is of interest by itself.
\end{abstract}

\section{Introduction}
Invertible neural networks based on coupling flows (\ARFINNs{}) are neural network architectures with invertibility by design \cite{PapamakariosNormalizing2019,KobyzevNormalizing2019}.
Endowed with the analytic-form invertibility and the tractability of the Jacobian, \ARFINNs{} have demonstrated their usefulness in various machine learning tasks such as
generative modeling \cite{DinhDensity2016a,KingmaGlow2018,OordParallel2018,KimFloWaveNet2019,ZhouDensity2019},
probabilistic inference \cite{pmlr-v89-bauer19a,WardImproving2019,pmlr-v70-louizos17a},
solving inverse problems \cite{ArdizzoneAnalyzing2018b}, and feature extraction and manipulation \cite{KingmaGlow2018,NalisnickHybrid2019,IzmailovSemisupervised2020,TeshimaFewshot2020}.
The attractive properties of \ARFINNs{} come at the cost of potential restrictions on the set of functions that they can approximate because they rely on carefully designed network layers.
To circumvent the potential drawback, a variety of layer designs have been proposed to construct \ARFINNs{} with high representation power, e.g.,
the affine coupling flow \cite{DinhNICE2014a,DinhDensity2016a,KingmaGlow2018,PapamakariosMasked2017a,KingmaImproved2016},
the neural autoregressive flow \cite{HuangNeural2018b,CaoBlock2019,HoFlow2018},
and the polynomial flow \cite{DBLP:conf/icml/JainiSY19},
each demonstrating enhanced empirical performance.

Despite the diversity of layer designs \cite{PapamakariosNormalizing2019,KobyzevNormalizing2019}, the theoretical understanding of the representation power of \ARFINNs{} has been limited.
Indeed, the most basic property as a function approximator, namely the \emph{universal approximation property} (or \emph{universality} for short) \cite{CybenkoApproximation1989}, has not been elucidated for \ARFINNs{}.
The universality can be crucial when \ARFINNs{} are used to learn an invertible transformation (e.g., feature extraction \cite{NalisnickHybrid2019} or independent component analysis \cite{TeshimaFewshot2020})
because, informally speaking, lack of universality implies that there exists an invertible transformation, even among well-behaved ones, that \ARFINN{} can never approximate, and it would render the model class unreliable for the task of function approximation.

To elucidate the universality of CF-INNs, we first prove a theorem to show the equivalence of the universality for certain diffeomorphism classes, which allows us to reduce the approximation of a general diffeomorphism to that of a much simpler one.
By leveraging this problem reduction, we show that \ARFINNs{} based on \emph{affine coupling flows} (ACFs; see Section~\ref{sec:preliminary}), one of the least expressive flow designs, are in fact universal approximators for a general class of diffeomorphisms.
The result can be interpreted as a convenient means to check the universality of a CF-INN: if the flow design can represent ACFs as special cases, then it is universal.

The difficulty in proving the universality of \ARFINNs{} lies in two complications. (1) Only function composition can be leveraged to make complex approximators (e.g., a linear combination is not allowed).
We overcome this complication by essentially decomposing a general diffeomorphism into much simpler ones, by using a structural theorem of differential geometry that elucidates the structure of a certain diffeomorphism group. Our equivalence theorem provides a way to take advantage of this technique implicitly.
(2) The flow layers tend to be inflexible due to the parametric restrictions.
As an extreme example, ACFs can only apply a uniform transformation along the transformed dimension, i.e., the parameter of the transformation cannot depend on the variable which undergoes the transformation. For ACFs, the reduction of the problem allows us to find an approximator with a clear outlook by approximating a step function.

\paragraph{Our contributions.}
Our contributions are summarized as follows.
\begin{enumerate}
  \item We present a theorem to show the equivalence of universal approximation properties for certain classes of functions. The result enables the reduction of the task of proving the universality for general diffeomorphisms to that for much simpler coordinate-wise ones.
  \item We leverage the result to show that some flow architectures, in particular even ACFs, can be used to construct a \ARFINN{} with the universality for approximating a fairly general class of diffeomorphisms.
  This result can be seen as a convenient criterion to check the universality of a \ARFINN: if the flow designs can reproduce ACF as a special case, it is universal.
\item As a corollary, we give an affirmative answer to a previously unsolved problem, namely the \emph{distributional universality} \cite{HuangNeural2018b,DBLP:conf/icml/JainiSY19} of ACF-based \ARFINNs{}.
\end{enumerate}
Our result is an interesting application of a deep theorem in differential geometry to investigate the representation power of a neural network architecture.

\section{Preliminary and goal}
\label{sec:preliminary}
In this section, we describe the models analyzed in this study, the notion of universality, and the goal of this paper.
We use \(\Re\) (resp. \(\Na\)) to represent the set of all real numbers (resp. positive integers). For a positive integer $n$, we define $[n]$ as the set $\{1,2,\ldots,n\}$.

\subsection{Invertible neural networks based on coupling flows}
Throughout the paper, we fix $d \in \Na$ and assume $d \geq 2$.
For a vector \(\x \in \ReD\) and \(k\in[d-1]\), we define \(\upto{k}{\x}\) as the vector \((x_1, \ldots, x_k)^\top \in \Re^k\) and \(\from{k}{\x}\) the vector \((x_{k+1}, \ldots, x_d)^\top \in \Re^{d-k}\).

\paragraph{Coupling flows.}
We define a coupling flow (\CF{}) \cite{PapamakariosNormalizing2019} $\AutoRegressive{k}{\tau}{\theta}$ by
\(\AutoRegressive{k}{\tau}{\theta}(\upto{k}{\x}, \from{k}{\x}) 
= (\upto{k}{\x}, \tau(\from{k}{\x}, \theta(\upto{k}{\x}))\),
where \(k\in[d-1]\), \(\theta\colon \Re^{k} \to \Re^l\) and \(\tau: \Re^{d-k} \times \Re^l \to \Re^{d-k}\) are maps, and \(\tau(\cdot, \theta(\bm{y}))\) is an invertible map for any $\bm{y}\in \Re^{k}$.

\paragraph{Affine coupling flows.}
One of the most standard types of \CFs{} is \emph{affine coupling flows} \cite{DinhDensity2016a, KingmaGlow2018, KingmaImproved2016, PapamakariosMasked2017a}.
We define an affine coupling flow \(\Aff{k}{s}{t}\colon \ReD \to \ReD\) by \[\Aff{k}{s}{t}(\upto{k}{\x}, \from{k}{\x}) = (\upto{k}{\x}, \from{k}{\x} \odot \exp(s(\upto{k}{\x})) + t(\upto{k}{\x})),\]
where \(k \in [d-1]\), $\odot$ is the Hadamard product, \(\exp\) is applied in an element-wise manner, and \(s,t:\mathbb{R}^{k}\to \mathbb{R}^{d-k}\) are maps typically parametrized by neural networks.

\paragraph{Single-coordinate affine coupling flow.}
Let $\ACFINNUniversalClass$ be a set of functions from $\R^{d-1}$ to \(\Re\). We define \emph{$\ACFINNUniversalClass$-single-coordinate affine coupling flows} by $\FSACFH:=\{\Aff{d-1}{s}{t}: s,t\in\ACFINNUniversalClass\}$, which is a subclass of ACFs.
It is the least expressive flow design appearing in this paper, but we show in Section~\ref{sec: main result 2} that it can form a \ARFINN{} with universality. We specify the requirements on \(\ACFINNUniversalClass\) later.

\paragraph{Invertible linear flows.}
We define the set of all affine transforms
by \(\FLin := \{\x \mapsto A\x + b: A \in \FGL, b \in \ReD\}\), where \(\FGL\) denotes the set of all regular matrices on \(\ReD\).

We consider the invertible neural network architectures constructed by composing flow layers:
\begin{definition}[\ARFINNs{}]
\label{def: INNM}
Let $\ARFINNFlow$ be a set consisting of invertible maps. We define the set of invertible neural networks based on \(\ARFINNFlow\) as
\[\INN{\ARFINNFlow} := \left\{ W_1\circ g_1\circ\cdots\circ W_n\circ g_n : \ n \in \Na, g_i\in \ARFINNFlow, W_i \in \FLin\right\}.\]
When \(\ARFINNFlow\) can represent the addition of a constant vector, we can obtain the same set of maps by replacing \(\FLin\) with \(\FGL\), which has been adopted by previous studies such as \citet{KingmaGlow2018}.
In fact, it is possible to use only 
the symmetric group $\mathfrak{S}_d$ that is the permutations of variables, instead of $\FLin$, 
when \(\ARFINNFlow\) contains \(\FSACFH\).
For details, see Appendix~\ref{sec:appendix:elementary matrix}.

\end{definition}

\subsection{Goal: the notions of universality and their relations}
Here, we clarify the notion of universality in this paper. First, we prepare some notation.
Let $p\in [1, \infty)$ and \(m, n \in \Na\). For a measurable mapping $f:\Re^m\to\Re^n$ and a subset $K\subset\Re^m$, we define
\[\LpKnorm{f} := \left(\int_K \Euclideannorm{f(x)}^p dx\right)^{1/p},\]
where $\Euclideannorm{\cdot}$ is the Euclidean norm of $\Re^n$.
We also define $\inftyKnorm{f}:=\sup_{x\in K}\Euclideannorm{f(x)}$.
\begin{definition}[\(L^p\)-/\(\sup\)-universality]
\label{def: Lp univ. approx.}
Let $\INNModelGeneric$ be a model which is a set of measurable mappings from $\Re^m$ to $\Re^n$.  
Let $p\in [1, \infty)$, and let $\mathcal{F}$ be a set of measurable mappings $f:U_f\rightarrow\Re^n$, where $U_f$ is a measurable subset of $\Re^m$ which may depend on $f$.
We say that 
$\INNModelGeneric$ is an \emph{$L^p$-universal approximator} or \emph{has the $L^p$-universal approximation property} for $\mathcal{F}$ if for any $f\in \mathcal{F}$, any $\varepsilon>0$, and any compact subset $K\subset U_f$, there exists $g\in \INNModelGeneric$ such that $\LpKnorm{f - g}<\varepsilon$.
We define the \emph{$\sup$-universality} analogously by replacing \(\LpKnorm{\cdot}\) with \(\inftyKnorm{\cdot}\).
\end{definition}
We also define the notion of distributional universality.
Distributional universality has been used as a notion of theoretical guarantee in the literature of normalizing flows, i.e., probability distribution models constructed using invertible neural networks \cite{KobyzevNormalizing2019}.
\begin{definition}[Distributional universality]
\label{def: dist. univ. approx.} 
Let $\INNModelGeneric$ be a model which is a set of measurable mappings from $\R^m$ to $\R^n$. 
We say that a model $\INNModelGeneric$ is a \emph{distributional universal approximator} or \emph{has the distributional universal approximation property} if, 
for any absolutely continuous\footnote{In this paper, we say a measure on the Euclidean space is \emph{absolutely continuous} when it is absolutely continuous with respect to the Lebesgue measure.} probability measure $\mu$ on $\R^m$ and any probability measure $\nu$ on $\R^n$, there exists a sequence $\{g_i\}_{i=1}^\infty\subset\INNModelGeneric$ such that $(g_i)_*\mu$ converges to $\nu$ in distribution as $i\rightarrow\infty$, where $(g_i)_*\mu:= \mu\circ g_i^{-1}$.
\end{definition}

If a model $\INNModelGeneric$ has the distributional universal approximation property, then it implies $\INNModelGeneric$ approximately transforms a known distribution, for example, the uniform distribution on $[0,1]^m$, into any probability measure $\mu$ on $\Re^n$, not only absolutely continuous but singular one.
There exists another convention that defines the distributional universality as a representation power for only absolutely continuous probability measures.
However, since absolutely continuous probability measures are dense in the set of all the probability measures, that convention is equivalent to ours. 
We include a proof for this fact in Lemma \ref{lem:appendix: abs aprox any} in Appendix \ref{appendix: from lp to dist}.

The different notions of universality are interrelated.
Most importantly, the \(L^p\)-universality for a certain function class implies the distributional universality (see Lemma~\ref{lem:body:distributional-universality}).
Moreover, if a model \(\INNModelGeneric\) is a \(\sup\)-universal approximator for \(\mathcal{F}\), it is also an $L^p$-universal approximator for \(\mathcal{F}\) for any $p \in [1, \infty)$.

\paragraph{Our goal}
Our goal is to elucidate the representation power of the \ARFINNs{} for some flow architectures \(\ARFINNFlow\) by proving the \(L^p\)-universality or \(\sup\)-universality of \(\INN{\ARFINNFlow}\) for a fairly large class of \emph{diffeomorphisms}, i.e., smooth invertible functions.
To prove universality, we need to construct a model \(g \in \INN{\ARFINNFlow}\) that attains the approximation error \(\varepsilon\) for given \(f\) and \(K\).

\section{Main results}
\label{sec:main-results}
In this section, we present the main results of this paper on the universality of \ARFINNs{}.
The first theorem provides a general proof technique to simplify the problem of approximating diffeomorphisms, and the second theorem builds on the first to show that the \ARFINNs{} based on the affine coupling are \(L^p\)-universal approximators.

\subsection{First main result: Equivalence of universal approximation properties}
Our first main theorem allows us to lift a universality result for a restricted set of diffeomorphisms to the universality for a fairly general class of diffeomorphisms by showing a certain equivalence of universalities.
By using the result to reduce the approximation problem, we can essentially circumvent the major complication in proving the universality of \ARFINNs{}, namely that only function composition can be leveraged to make complex approximators (e.g., a linear combination is not allowed).

First, we define the following classes of invertible functions. Our main theorem later reveals an equivalence of \(L^p\)-universality/\(\sup\)-universality for these classes.
\begin{definition}[\(C^2\)-diffeomorphisms: $\CtwoDomainDiff$]
\label{def: D}
We define $\CtwoDomainDiff$ as the set of all $C^2$-diffeomorphisms $f:U_f\rightarrow {\rm Im}(f)\subset\ReD$ 
, where $U_f \subset \ReD$ is an open set $C^2$-diffeomorphic to $\ReD$, which may depend on $f$.
\end{definition}
\begin{definition}[Triangular transformations: $\Triangular$]
\label{def: T}
We define $\Triangular$ as the set of all 
$C^\infty$-\emph{increasing triangular} mappings from \(\ReD\) to \(\ReD\). Here, a mapping $\tau=(\tau_1, \ldots, \tau_d):\ReD\to\ReD$ is increasing triangular if each \(\tau_k(\x)\) depends only on \(\upto{k}{\x}\) and is strictly increasing with respect to \(x_k\).
\end{definition}
\begin{definition}[Single-coordinate transformations: \(\CrOneDimTriangular\)]
\label{def: Ts}
We define $\CrOneDimTriangular$ as the set of all compactly-supported $C^r$-diffeomorphisms $\tau$ satisfying $\tau(\x)=(x_1, \ldots, x_{d-1}, \tau_d(\x))$, i.e., those which alter only the last coordinate. In this article, only \(r = 0, 2, \infty\) appear, and we mainly focus on \(\CinftyOneDimTriangular (\subset\Triangular)\).
Here, a bijection $\tau:\ReD \rightarrow\ReD$ is compactly supported if $\tau = \Identity$ outside some compact set.
\end{definition}
Among the above classes of invertible functions, $\CtwoDomainDiff$ is our main approximation target, and it is a fairly large class: it contains any \(C^2\)-diffeomorphism defined on the entire \(\ReD\), an open convex set, or more generally a star-shaped open set.
The class \(\Triangular\) relates to the distributional universality as we will see in Lemma~\ref{lem:body:distributional-universality}.
The class \(\OneDimTriangular\) is a much simpler class of diffeomorphisms that we use as a stepladder for showing the universality for \(\CtwoDomainDiff\).

Now we are ready to state the first main theorem. It reveals an equivalence among the universalities for \(\CtwoDomainDiff\), \(\Triangular\), and \(\OneDimTriangular\), under mild regularity conditions. We can use the theorem to lift up the universality for \(\OneDimTriangular\) to that for \(\CtwoDomainDiff\).
\begin{theorem}[Equivalence of Universality]
Let $p\in[1, \infty)$ and let $\ARFINNFlow$ be a set of invertible functions.
\begin{enumerate}[(A)]
    \item \label{main thm: A} If all elements of $\ARFINNFlow$ are piecewise $C^1$-diffeomorphisms, then the $L^p$-universal approximation properties of $\INN{\ARFINNFlow}$ for $\CtwoDomainDiff$, $\Triangular$ and $\OneDimTriangular$ are all equivalent.
    \item \label{main thm: B} If all elements of $\ARFINNFlow$ are locally bounded,
    then the $\sup$-universal approximation properties of $\INN{\ARFINNFlow}$ for $\CtwoDomainDiff$, $\Triangular$ and $\OneDimTriangular$ are all equivalent.
\end{enumerate}
\label{thm:body:diffeo-universal-equivalences}
\label{theorem:main:1}
\end{theorem}
The proof is provided in Appendix~\ref{sec:appendix:universality-proof}.
For the definitions of the piecewise $C^1$-diffeomorphisms and the locally bounded maps, see Appendix~\ref{sec:appendix:piecewise diffeo}.
The regularity conditions in (\ref{main thm: A}) and (\ref{main thm: B}) assure that function composition within \(\ARFINNFlow\) is compatible with approximations (see Appendix~\ref{appendix: compatibility of approximation and composition} for details), and they are usually satisfied, e.g., continuous maps are locally bounded.

If one of the two universality properties in Theorem~\ref{thm:body:diffeo-universal-equivalences} is satisfied, the model is also a distributional universal approximator.
Let \(p \in [1, \infty)\), and we have the following.
\begin{lemma}
\label{lem:body:distributional-universality}
An \(L^p\)-universal approximator for \(\Triangular\) is a distributional universal approximator.
\end{lemma}
Since \(\sup\)-universality implies \(L^p\)-universality, Lemma~\ref{lem:body:distributional-universality} can be combined with both cases of (\ref{main thm: A}) and (\ref{main thm: B}) in Theorem~\ref{thm:body:diffeo-universal-equivalences}.
The proof is based on the existence of a triangular map connecting two absolutely continuous distributions \cite{BogachevTriangular2005}. See Appendix~\ref{appendix: from lp to dist} for details.
Note that the previous studies \cite{DBLP:conf/icml/JainiSY19,HuangNeural2018b} have discussed the distributional universality of some flow architectures essentially via showing the \(\sup\)-universality for \(\Triangular\).
Lemma~\ref{lem:body:distributional-universality} clarifies that the weaker notion of \(L^p\)-universality is sufficient for the distributional universality, which can also apply to the case (\ref{main thm: A}) in Theorem~\ref{thm:body:diffeo-universal-equivalences}.

\paragraph{Application to previously proposed \ARFINN{} architectures.}
Theorem~\ref{thm:body:diffeo-universal-equivalences} can upgrade a previously known \(\sup\)-universality for \(\Triangular\) of a \ARFINN{} architecture to that for \(\CtwoDomainDiff\).
As examples, \emph{deep sigmoidal flows} (DSF; a version of neural autoregressive flows \cite{HuangNeural2018b}) and \emph{sum-of-squares polynomial flows} (SoS; \cite{DBLP:conf/icml/JainiSY19}) can both yield \ARFINNs{} with the \(\sup\)-universal approximation property for \(\CtwoDomainDiff\).
We provide the proof in Appendix~\ref{appendix:sec:examples}.
See Table~\ref{tbl:example-architectures} for a summary of the results.
\begin{table}
  \caption{
  \ARFINN{} instances analyzed in this work (
  \emph{Model}: the considered \ARFINN{} architecture.
  \emph{Flow type}: the flow layer architecture.
  \emph{Universality (this)}: the universal approximation property that this work has shown.
  \emph{Universality (prev.)}: previously claimed universal approximation property.
  )
  Our proof techniques are easy to apply to analyze the universality of various \ARFINN{} architectures.
  }
  \label{tbl:example-architectures}
  \centering
  \begin{tabular}{lllll}
    \toprule
    Model  & Flow type     & Universality (this) &  Universality (prev.)  \\
    \midrule
    \(\INN{\FSACFH}\) & Affine coupling \cite{DinhDensity2016a, KingmaGlow2018, KingmaImproved2016, PapamakariosMasked2017a} & \(L^p\)-universal & - \\
    \(\INN{\FDSF}\) & Deep sigmoidal flow \cite{HuangNeural2018b} & \(\sup\)-universal & Distributional \cite{HuangNeural2018b}  \\
    \(\INN{\FSoS}\) & Sum-of-squares polynomial flow \cite{DBLP:conf/icml/JainiSY19} & \(\sup\)-universal & Distributional \cite{DBLP:conf/icml/JainiSY19} \\
    \bottomrule
  \end{tabular}

\end{table} See Section~\ref{sec:related-work:normalizing-flow-guarantee} for a comparison with previous theoretical analyses on normalizing flows.

\subsection{Second main result: \(L^p\)-universal approximation property of $\INN{\FSACFH}$}
\label{sec: main result 2}
Our second main theorem reveals the \(L^p\)-universality of $\INN{\FSACFH}$ for \(\CzeroOneDimTriangular\) (hence for \(\CinftyOneDimTriangular\)), which can be combined with Theorem~\ref{thm:body:diffeo-universal-equivalences} to show its \(L^p\)-universality for \(\CtwoDomainDiff\).
We define \(\CcinftyRDminus\) as the set of all compactly-supported \(\Cinfty\) maps from \(\ReDminus\) to \(\Re\).
\begin{theorem}[$L^p$-universality of \(\INN{\FSACFH}\)]
\label{prop:body:acfinn-Lp}\label{theorem:main:2}
Let \(p \in [1, \infty)\).
Assume \(\ACFINNUniversalClass\) is a \(\sup\)-universal approximator for \(\CcinftyRDminus\) and that it consists of piecewise \(C^1\)-functions.
Then, \(\INN{\FSACFH}\) is an \(L^p\)-universal approximator for \(\CzeroOneDimTriangular\).
\end{theorem}
We provide a proof in Appendix~\ref{appendix: theorem 2 proof}.
For the definition of piecewise \(C^1\)-functions, see Appendix~\ref{sec:appendix:piecewise diffeo}.
Theorem~\ref{theorem:main:2} can be combined with Theorem~\ref{theorem:main:1} to show that \(\INN{\FSACFH}\) is an \(L^p\)-universal approximator for \(\CtwoDomainDiff\).
Examples of \(\ACFINNUniversalClass\) satisfying the condition of Theorem~\ref{prop:body:acfinn-Lp} include multi-layer perceptron models with the \emph{rectifier linear unit} (ReLU) activation \cite{LeCunDeep2015} and a linear-in-parameter model with smooth universal kernels \cite{MicchelliUniversal2006a}.
The result can be interpreted as a convenient criterion to check the universality of a \ARFINN{}: if the flow architecture \(\ARFINNFlow\) contains \ACFs{} (or even just \(\FSACFH\) with sufficiently expressive \(\ACFINNUniversalClass\)) as special cases, then \(\INN{\ARFINNFlow}\) is an \(L^p\)-universal approximator for \(\CtwoDomainDiff\).

By combining Theorem~\ref{theorem:main:1}, Theorem~\ref{theorem:main:2}, and Lemma~\ref{lem:body:distributional-universality}, we can affirmatively answer a previously unsolved problem \citep[p.13]{PapamakariosNormalizing2019}: the distributional universality of \ARFINN{} based on \ACFs{}.
\begin{theorem}[Distributional universality of \(\INN{\FSACFH}\)]
Under the conditions of Theorem~\ref{theorem:main:2}, \(\INN{\FSACFH}\) is a distributional universal approximator.
\label{cor:body:acfinn-dist-universal}
\label{cor:body:afinn-distribution-universal}
\end{theorem}

\paragraph{Implications of Theorem~\ref{theorem:main:2} and Theorem~\ref{cor:body:acfinn-dist-universal}.}
Theorem~\ref{theorem:main:2} implies that, if \(\ARFINNFlow\) contains \(\FSACFH\) as special cases, then \(\INN{\ARFINNFlow}\) is an \(L^p\)-universal approximator for \(\CtwoDomainDiff\). In light of Theorem~\ref{cor:body:acfinn-dist-universal}, it is also a distributional universal approximator, hence we can confirm the theoretical plausibility for using it for normalizing flows. Such examples of \(\ARFINNFlow\) include the \emph{nonlinear squared flow} \cite{ZieglerLatent2019a}, \emph{Flow++} \cite{HoFlow2018}, the \emph{neural autoregressive flow} \cite{HuangNeural2018b}, and the \emph{sum-of-squares polynomial flow} \cite{DBLP:conf/icml/JainiSY19}. 
The result may not immediately apply to the typical \emph{Glow} \cite{KingmaGlow2018} models for image data that use the 1x1 invertible convolution layers and convolutional neuralnetworks for the coupling layers.
However, the Glow architecture for non-image data \cite{ArdizzoneAnalyzing2018b,TeshimaFewshot2020} can be interpreted as \(\INN{\ARFINNFlow}\) with ACF layers, hence it is both an \(L^p\)-universal approximator for \(\CtwoDomainDiff\) and a distributional universal approximator.

\section{Proof outline}
In this section, we outline the proof ideas of our main theorems to provide an intuition for the constructed approximator and derive reusable insight for future theoretical analyses.

\subsection{Proof outline for Theorem~\ref{thm:body:diffeo-universal-equivalences}}
\label{subsection: outline of the proof of theorem 1}

Here, we outline the equivalence proof of Theorem~\ref{theorem:main:1}.
For details, see Appendix~\ref{sec:appendix:universality-proof}.
Since we have $\CinftyOneDimTriangular\subset \Triangular\subset \CtwoDomainDiff$, it is sufficient to prove that the universal approximation properties for $S^\infty$ implies that for $\CtwoDomainDiff$.
Note that the proofs do not change for \(L^p\)-universality and \(\sup\)-universality.

Therefore, we focus on describing the reduction from \(\CtwoDomainDiff\) to \(\CinftyOneDimTriangular\).
Since the approximation of $\CtwoOneDimTriangular$ can be reduced to that of \(\CinftyOneDimTriangular\) by a standard mollification argument (see Appendix~\ref{sec:appendix:Dc2 to Sinfty}), we show a reduction from \(\CtwoDomainDiff\) to \(\CtwoOneDimTriangular\):
\begin{theorem}
\label{thm: equivalence S2-D2}
For any element $f\in\CtwoDomainDiff$ and compact subset $K\subset U_f$, there exist $n\in \Na$, $W_1,\dots, W_n \in \FLin$, and $\tau_1,\dots,\tau_n\in \CtwoOneDimTriangular$ such that
$f(x)=W_1\circ\tau_1\circ\cdots\circ W_n\circ\tau_n(x)$ for all $x\in K$.
\end{theorem}
Behind the scenes, Theorem~\ref{thm: equivalence S2-D2} reduces \(\CtwoDomainDiff\) to \(\CtwoOneDimTriangular\) in four steps:
\begin{equation*}\begin{aligned}
\CtwoDomainDiff
\rightsquigarrow \DcRD
\rightsquigarrow \text{Flow endpoints}
\rightsquigarrow \text{nearly-$\Identity$}
\rightsquigarrow \CtwoOneDimTriangular
\end{aligned}\end{equation*}
Here, \(A \rightsquigarrow B\) (\(A\) \emph{is reduced to} \(B\)) indicates that the universality for \(A\) follows from that for \(B\), and \(\Identity\) denotes the identity map.
We explain each reduction step in the below.

\textit{From \(\CtwoDomainDiff\) to \(\DcRD\). }
We consider a special subset $\DcRD\subset \CtwoDomainDiff$, which is the group of \emph{compactly-supported} $C^2$-diffeomorphisms on \(\ReD\) whose group operation is functional composition.
Here, a bijection $f:\ReD \rightarrow\ReD$ is compactly supported if $f = \Identity$ outside some compact set.
Proposition~\ref{prop: extension lemma} below reduces the problem of the universality for $\CtwoDomainDiff$ to that for $\DcRD$.
\begin{proposition}
\label{prop: extension lemma}
For any $f\in\CtwoDomainDiff$ and any compact subset $K\subset U_f$, there exist $h\in \DcRD$, $W \in \FLin$, such that for all $\x\in K$, $f(\x)=W\circ h(\x)$.
\end{proposition}

\textit{From \(\DcRD\) to flow endpoints. }
In order to construct an approximation for the elements of $\CtwoDomainDiff$, we devise its subset that we call the \emph{flow endpoints}.
A flow endpoint is an element of $\DcRD$ which can be represented as $\phi(1)$ using an ``additive'' continuous map $\phi:[0,1]\rightarrow \DcRD$ with $\phi(0)=\Identity$. Here, ``additivity'' means $\phi(s)\circ\phi(t)=\phi(s+t)$ for any $s,t\in[0,1]$ with $s+t\in[0,1]$.
This additivity will be later used to decompose a flow endpoint into a composition of some mildly-behaved fragments of the flow map.
Note that we equip $\DcRD$ with the Whitney topology \citep[Proposition~1.7.(9)]{HallerGroups1995} to define the continuity of the map $\phi$.
The importance of the flow endpoints lies in the following lemma that we prove in Appendix~\ref{sec:appendix:Dc2 to nearly-Id}:
\begin{lemma}
\label{lem:body:flow-approximation}
Any element in \(\DcRD\) can be represented as a finite composition of flow endpoints.
\end{lemma}
Lemma~\ref{lem:body:flow-approximation} is essentially due to Fact~\ref{fact: simplicity}, which is the following structure theorem in differential geometry attributed to Herman, Thurston \cite{ThurstonFoliations1974}, Epstein \cite{Epsteinsimplicity1970}, and Mather \cite{MatherCommutators1974, MatherCommutators1975}:
\begin{fact}
\label{fact: simplicity}
The group $\DcRD$ is simple, i.e., any normal subgroup $H \subset \DcRD$ is either $\{\Identity\}$ or $\DcRD$.
\end{fact}

\textit{From flow endpoints to nearly-$\Identity$. }
The flow endpoints in $\DcRD$ can be decomposed into "nearly-$\Identity$" elements in $\DcRD$ by leveraging its additivity property, as in the following proposition.
Let $\opnorm{\cdot}$ denote the operator norm.
\begin{proposition}
\label{prop:composition of deformations of Id}
For any $f\in \DcRD$, there exist finite elements $g_1,\dots,g_r\in\DcRD$ such that $f=g_1\circ\cdots\circ g_r$ and $\sup_{x\in\ReD} \opnorm{\Jac{g_i}(x)-I}<1$, where $\Jac{g_i}$ is the Jacobian of $g_i$.
\end{proposition}
Proposition~\ref{prop:composition of deformations of Id} leverages the continuity of the flows with respect to the Whitney topology of \(\DcRD\): $\phi(1/n)$ uniformly converges to the identity map both in its values and \emph{its Jacobian} when $n\rightarrow\infty$.
Thus, any flow endpoint $\phi(1)$ can be represented by an $n$-time composition of $\phi(1/n)$ each of which is close to identity (nearly-$\Identity$) when $n$ is sufficiently large.

\textit{From nearly-$\Identity$ to \(\CtwoOneDimTriangular\). }
The nearly-$\Identity$ elements, $g\in \DcRD$ in Proposition \ref{prop:composition of deformations of Id}, can be decomposed into elements of $\CtwoOneDimTriangular$ and permutation matrices:
\newcommand{\twodimmat}[4]{\begin{pmatrix}#1 & #2 \\ #3 & #4\end{pmatrix}}
\newcommand{\twodimvec}[2]{\begin{pmatrix}#1 \\ #2\end{pmatrix}}
\newcommand{\twodimargument}[2]{\mleft(\begin{matrix}#1 \\ #2\end{matrix}\mright)}
\renewcommand\thesubfigure{\roman{subfigure}}

\begin{figure}[t]

\begin{minipage}[c]{1.0\linewidth}
\begin{minipage}[c]{0.3\linewidth}
  \includegraphics[keepaspectratio, width=\linewidth]{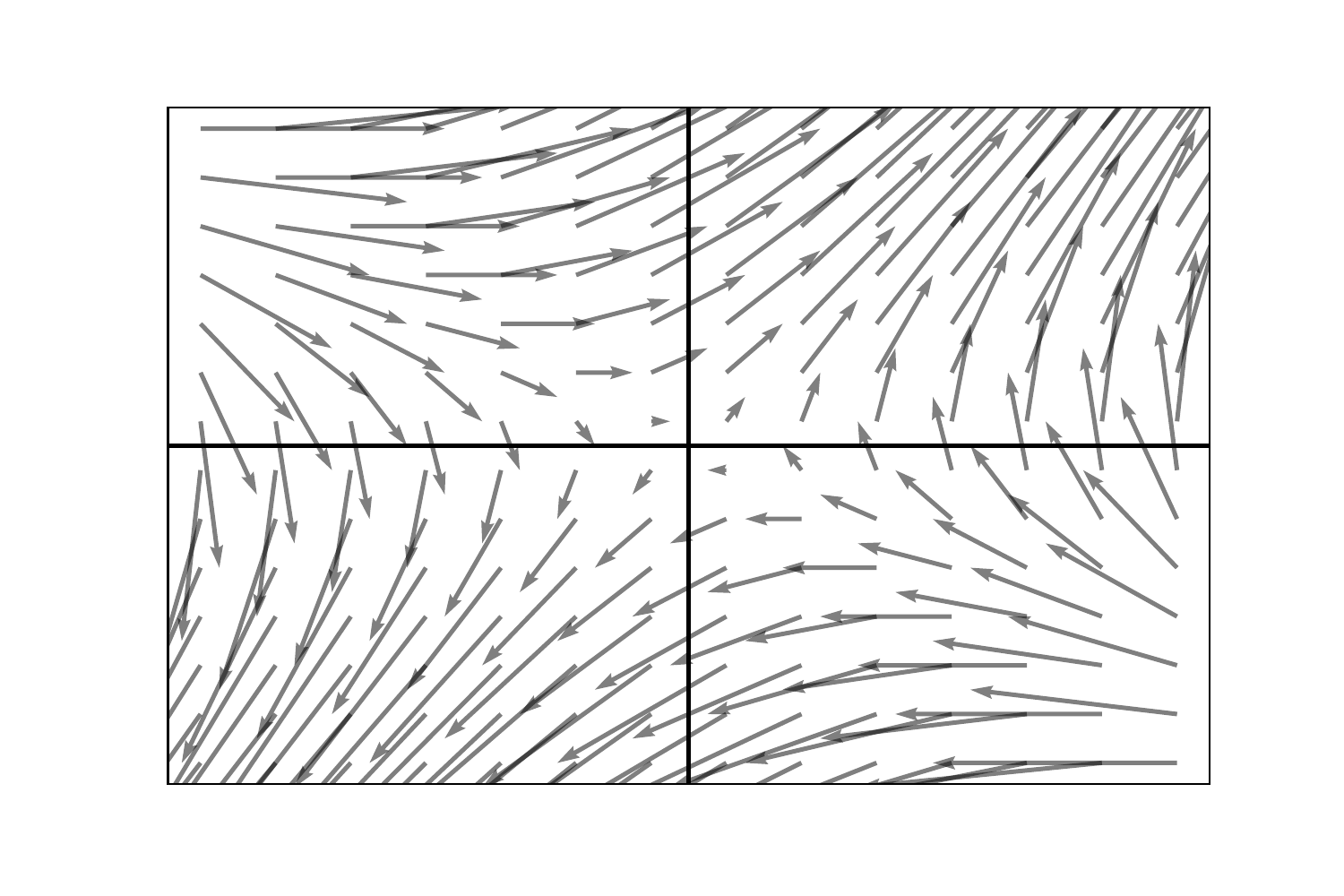}
\end{minipage}\hfill \(=\) \hfill
\begin{minipage}[c]{0.3\linewidth}
  \includegraphics[keepaspectratio, width=\linewidth]{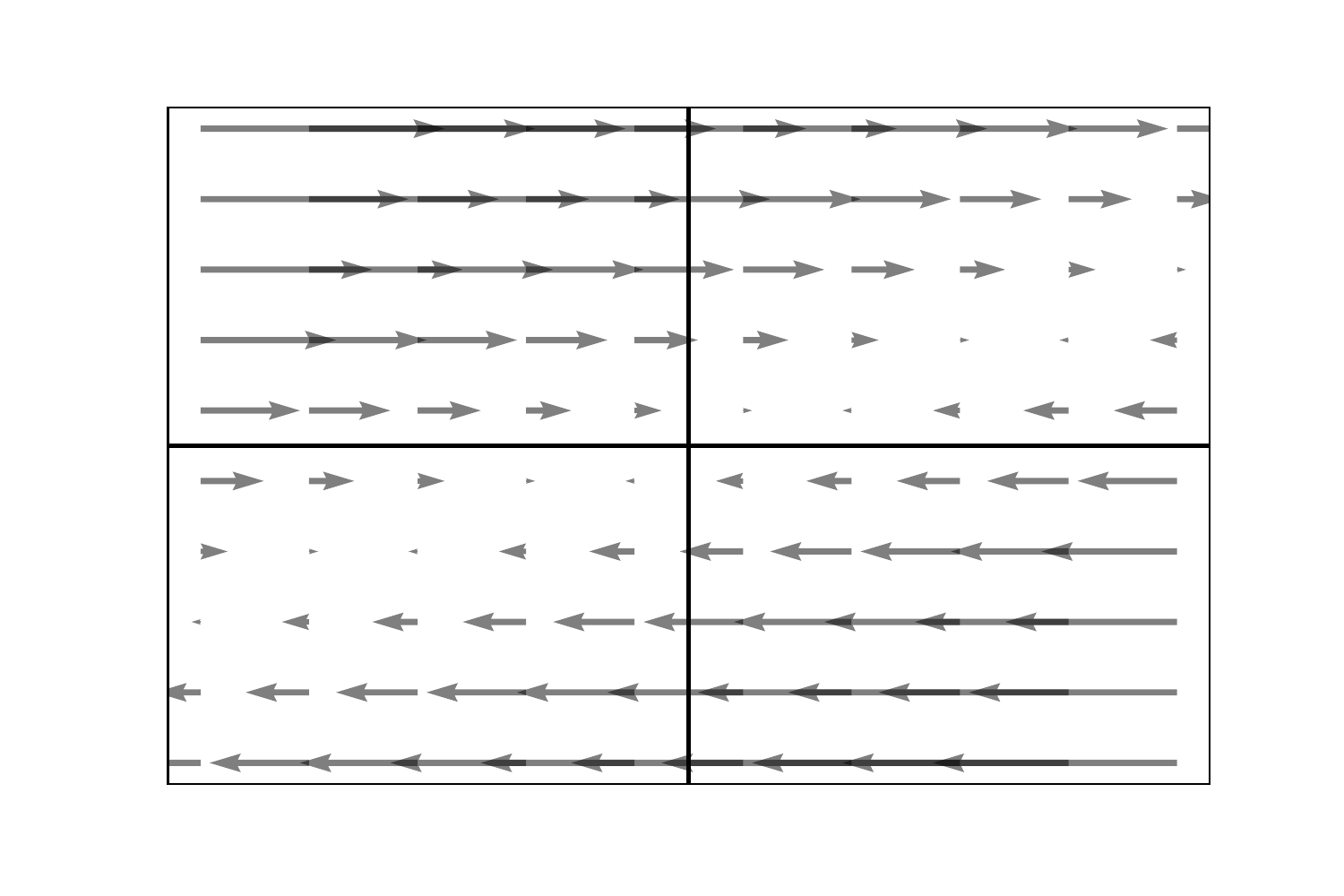}
\end{minipage}\hfill \(\circ\) \hfill
\begin{minipage}[c]{0.3\linewidth}
  \includegraphics[keepaspectratio, width=\linewidth]{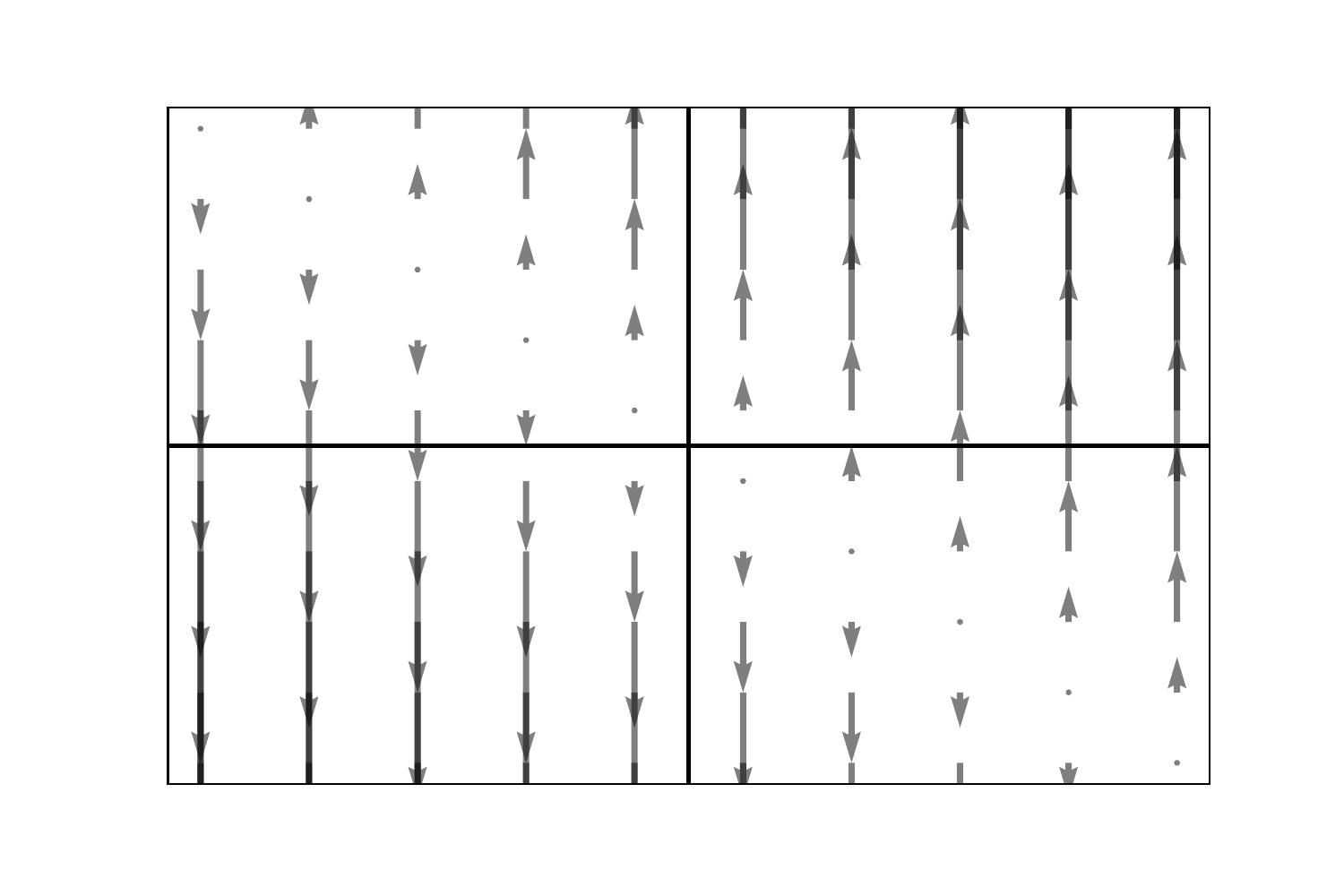}
\end{minipage}\hfill
\end{minipage}

\begin{minipage}[c]{1.0\linewidth}\small{}
\begin{minipage}[c]{0.3\linewidth}\centering{}
  \(f = \twodimmat{a}{b}{c}{d} = f_2 \circ f_1\)
\end{minipage}\hfill
\begin{minipage}[c]{0.38\linewidth}\centering{}
  \(f_2\twodimargument{x_1}{y_2} = \twodimvec{a x_1 + b (\frac{y_2 - c x_1}{d})}{y_2}\)
\end{minipage}\hfill
\begin{minipage}[c]{0.3\linewidth}\centering{}
  \(f_1\twodimargument{x_1}{x_2} = \twodimvec{x_1}{c x_1 + d x_2}\)
\end{minipage}\hfill
\end{minipage}

\caption{
  A nearly-$\Identity$ transformation \(f\) can be decomposed into coordinate-wise ones (\(f_1\) and \(f_2\): realized by \(\CtwoOneDimTriangular\) and permutations).
  The arrows indicate the transportation of the positions.
  A general nonlinear \(f\) can be analogously decomposed by Proposition~\ref{prop:triangle decomposition} when \(f\) satisfies certain conditions.
}
\label{fig:linear-flow-decomp}
\end{figure}

 \begin{proposition}
\label{prop:triangle decomposition}
For any $g\in \DcRD$ with $\sup_{x\in\ReD} \opnorm{\Jac{g}(x)-I}<1$, there exist $d$ elements $\tau_1,\dots,\tau_d\in \CtwoOneDimTriangular$ and permutation matrices $\sigma_1,\dots,\sigma_d$ such that
\[g=\sigma_1\circ\tau_1\circ\cdots\circ\sigma_d\circ\tau_d.\]
\end{proposition}
The machinery of this decomposition is illustrated in Figure~\ref{fig:linear-flow-decomp}.

\subsection{Proof outline for Theorem~\ref{theorem:main:2}}
\label{subsec: approximation via ACF}
\def \MyFigureACFINNbasesize {0.2}
\def \MyFigureACFINNbaseVdistance {0.5cm}
\def \MyFigureACFINNbaseHdistance {2cm}
\begin{figure}[t]

\begin{minipage}[c]{1.0\linewidth}
\begin{minipage}[c]{1.0\linewidth}
\centering{}
\begin{tikzpicture}
\node (1) at (0,0) {\includegraphics[keepaspectratio, width=\MyFigureACFINNbasesize\linewidth]{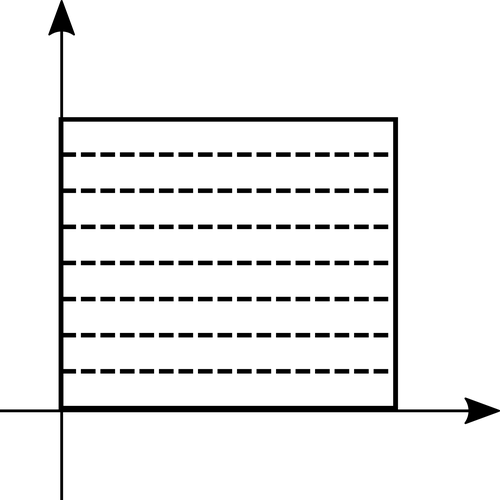}};
\node[right=\MyFigureACFINNbaseHdistance of 1] (2) {\includegraphics[keepaspectratio, width=\MyFigureACFINNbasesize\linewidth]{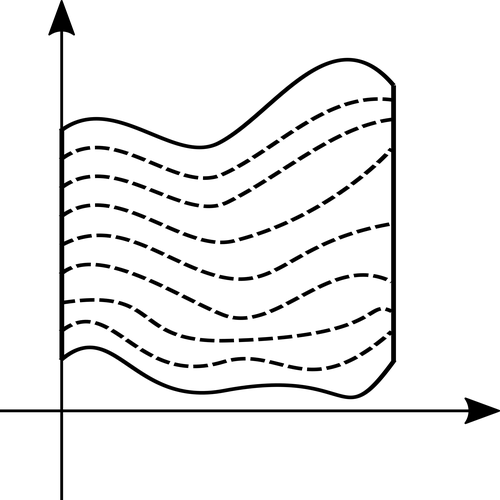}};
\node[right=\MyFigureACFINNbaseHdistance of 2] (3) {\includegraphics[keepaspectratio, width=\MyFigureACFINNbasesize\linewidth]{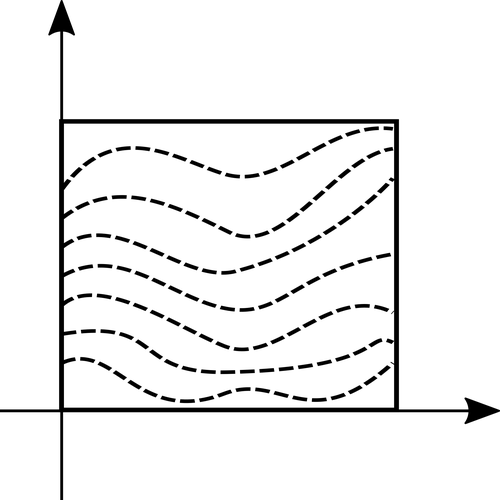}};
\node[below=\MyFigureACFINNbaseVdistance of 1] (4) {\includegraphics[keepaspectratio, width=\MyFigureACFINNbasesize\linewidth]{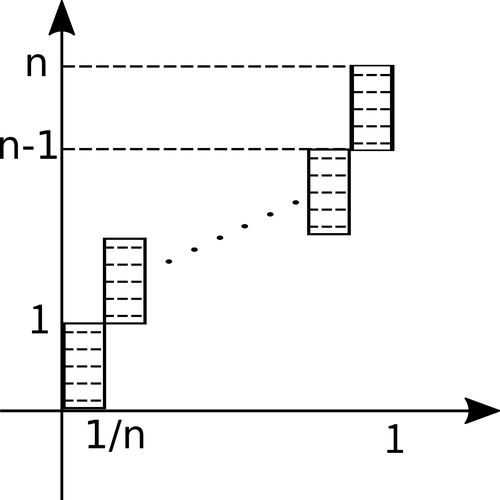}};
\node[below=\MyFigureACFINNbaseVdistance of 3] (5) {\includegraphics[keepaspectratio, width=\MyFigureACFINNbasesize\linewidth]{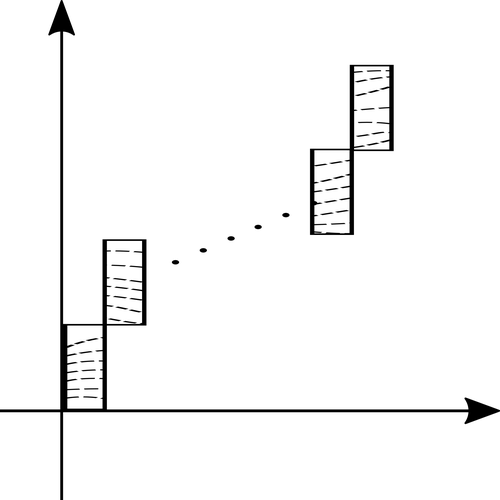}};

\draw[->,thick] (1) -- (2) node[midway, above] {\(f\)};
\draw[->,thick] (2) -- (3) node[midway, above] {\(\FACF\) (Step~1)};
\draw[->,thick, transform canvas={xshift=-1.5em}] (1) -- (4) node[midway,right] {\(\psinstar\) (Step~2)};
\draw[->,thick, transform canvas={xshift=-1.5em}] (3) -- (5) node[midway,right] {\(\psimstar\) (Step~2)};
\draw[->,thick, transform canvas={yshift=3em}] (4) -- (5)
node[midway,above] {\((x, y) \mapsto (x, v_n(y))\) \ (Step~3)}
node[midway,below, yshift=-0.05em, text width=3cm, xshift=-1.4cm]{
\begin{equation*}\begin{aligned}
&\exists g_1, g_2, g_3 \in \FACFINN:\\
&g_1  \simeq \psinstar,\  g_2 \simeq (x, v_n(y)),\ g_3\simeq (\psimstar)^{-1} \\[2pt]
&\hspace{-3pt}\Longrightarrow{}f \simeq g_3\circ g_2\circ g_1 \quad \text{(Steps~4, 5)}\\
\end{aligned}\end{equation*}
};
\end{tikzpicture}
\end{minipage}

\begin{minipage}[c]{1.0\linewidth}
\caption{Illustration of the proof technique for the \(L^p\)-universal approximation property of \ACFINN{} for \(\CzeroOneDimTriangular\).
The symbol \(\simeq\) indicates approximation to arbitrary precision.
}
\label{fig:acf-construction}
\end{minipage}
\end{minipage}

\end{figure}

 Here, we give the proof outline of Theorem~\ref{theorem:main:2}.
For details, see Appendix~\ref{appendix: theorem 2 proof}.
The main difficulty in constructing the approximator is the restricted functional form of \ACFs{}.
However, the problem reduction by Theorem~\ref{theorem:main:1} allows us to construct an approximator by approximating a step function.

For illustration, we only describe the case for $d=2$ and $K \subset [0,1]^2$. For complete proof of Theorem~\ref{theorem:main:2}, see Appendix~\ref{appendix: theorem 2 proof}.
Let $f(x, y)=(x,u(x,y))$ be the target function, where $u(\cdot,y)$ is a continuous function that is strictly increasing for each $y$ (i.e., $f \in \CzeroOneDimTriangular$). For the compact set $K \subset [0, 1]^2 \subset\Re^2$, we find $g\in\INN{\FSACFH}$ arbitrarily approximating $f$ on $K$ as follows (Figure~\ref{fig:acf-construction}).
\begin{itemize}
    \item[Step~1.] \textbf{Align the image into the square:} First, without loss of generality, we may assume that the image $f([0, 1]^2)$ is again $[0, 1]^2$.
    Indeed, we can align the image so that $u(x,1)=1$ and $u(x,0)=0$ for all $x\in[0,1]$ by using only an ACF $\Aff{1}{s}{t}$ with continuous $s$ and $t$, which can be approximated by \(\FSACFH\).
    \item[Step~2.] \textbf{Slice the squares and stagger the pieces:} We consider an imaginary ACF $\psi_n^*:=\Psi_{1,1,t_n}$ defined using a discontinuous step function $t_n:=\sum_{k=0}^nk\Indicator{[k/n,(k+1)/n)}$.
    The map $\psi_n^*$ splits $[0, 1]^2$ into pieces and staggers them so that a coordinate-wise independent transformation (e.g., $v_n$ in Step~3), which is uniform along the $x$-axis, can affect each piece separately.
    \item[Step~3.] \textbf{Express $f$ by a coordinate-wise independent transformation:}
    We construct a continuous increasing function $v_n:\R\rightarrow\R$ such that for $y\in[k,k+1)$, $v_n(y) = u(k/n, y)+k$ ($k=0,\dots,n-1$).
    A direct computation shows that $\tilde{f}_n:=(\psi_n^*)^{-1}\circ(\cdot,v_n(\cdot))\circ\psi_n^*$ arbitrarily approximates $f$ on $[0,1]^2$ if we increase $n$.
    We take a sufficiently large $n$.
    \item[Step~4.] \textbf{Approximate the coordinate-wise independent transformation \(v_n\)}:
    We find an element of $\INN{\FSACFH}$ sufficiently approximating $(\cdot,v_n(\cdot))$ on $[0,1]\times[0,n]$.
    This is realized based on a lemma that we can construct an approximator for any element of $\CzeroOneDimTriangular$ of the form $(x, y) \mapsto (x,v(y))$
    on any compact set in $\Re^2$.
    \item[Step~5.] \textbf{Approximate \(\psi_n^*\) and combine the approximated constituents to approximate $\tilde{f}_n$:}
    We can also approximate $\psi_n^*$ and its inverse by \ACFs{} based on the universality of $\ACFINNUniversalClass$.
    Finally, composing the approximated constituents gives an approximation of $f$ on $[0,1]^2$ with arbitrary precision (see Appendix \ref{appendix: compatibility of approximation and composition}).
\end{itemize}

\section{Related work and discussions\label{sec:related-work}}
In this section, we relate the contribution of this work to the literature on the representation power of invertible neural networks.

\subsection{Relation to previous theoretical analyses for normalizing flow models}\label{sec:related-work:normalizing-flow-guarantee}
The distributional universality of \emph{normalizing flows} constructed using \ARFINNs{} has been addressed in previous studies such as \cite{DBLP:conf/icml/JainiSY19,HuangNeural2018b}.
Previously proposed architectures with distributional universality include the neural autoregressive flows \cite{HuangNeural2018b} and the sum-of-squares polynomial flows \cite{DBLP:conf/icml/JainiSY19}.
Our findings elucidate the much stronger universalities of these architectures, namely the \(\sup\)-universality for \(\CtwoDomainDiff\), which enhances the reliability of these models in the tasks where function approximation rather than distribution approximation is crucial, e.g., feature extraction \cite{NalisnickHybrid2019,TeshimaFewshot2020}.

\subsection{Theoretical guarantee for other invertible neural network architectures}

\paragraph{One-dimensional case.}
In the one-dimensional case (\(d = 1\)), strict monotonicity is a necessary and sufficient condition for a function to be invertible.
In this case, there have been a few invertible neural network architectures with \(\sup\)-universality for the set of all homeomorphisms on \(\Re\), e.g., \emph{monotonic networks} \cite{SillMonotonic1998} and \emph{rational quadratic splines} \cite{DurkanNeural2019}.
These models complement \ARFINNs{} in that they provide an invertible neural network only in the one-dimensional case, whereas the latter can be defined only in the multi-dimensional case.

\paragraph{Relation to examples of functions that cannot be approximated by NODEs.}
Neural ordinary differential equations (NODEs) \citep{ChenNeural2018a,DupontAugmented2019a} can be considered as another design of invertible flow layers different from CFs.
\citet{ZhangApproximation2019a} formulated its Theorem~1 to show that NODEs are not universal approximators by presenting a function that a NODE cannot approximate.
The existence of this counterexample does not contradict our result because our approximation target \(\CtwoDomainDiff\) is different from the function class considered in \citet{ZhangApproximation2019a}: the class in \citet{ZhangApproximation2019a} can contain discontinuous maps whereas the elements of \(\CtwoDomainDiff\) are smooth and invertible.
Also, in Proposition~\ref{prop: extension lemma}, we cap an affine transformation (realizable by \(\INN{\ARFINNFlow}\)) on top of the target function to reduce the approximation of \(\CtwoDomainDiff\) to that of \(\DcRD\).
Such an affine transformation may enhance the approximation capacity by allowing a certain set of transformations, e.g., coordinate-wise sign flipping.

\subsection{The strength of the representation power of \(\INN{\FSACFH}\)}
\label{discussion of EACFINN}
In this study, we showed the \(L^p\)-universal approximation property of $\INN{\FSACFH}$. 
While the \(L^p\)-universality is likely to suffice for developing probabilistic risk bounds for machine learning tasks \cite{LinResNet2018,SuzukiAdaptivity2018} and for showing distributional universality,
whether $\INN{\FSACFH}$ is a \(\sup\)-universal approximator for \(\CtwoDomainDiff\) remains an open question.
Our conjecture is negative due to the following theoretical observation.
The \(\sup\)-universality requires a precise approximation uniformly everywhere while the \(L^p\)-universality can allow an approximation error on negligible regions.
As described in Section~\ref{subsec: approximation via ACF}, we used a smooth approximation of step functions to show the \(L^p\)-universality of $\INN{\FSACFH}$.
Intuitively, approximating the step functions and composing them can accumulate errors around the discontinuity points, so that it can retain the \(L^p\)-universality but it can affect the \(\sup\)-universality.
Since the step functions are devised to bypass the uniformity of the transformation by ACFs,
we conjecture that the difficulty is intrinsic and a \(\sup\)-universality is unlikely to hold for \(\INN{\FSACFH}\).

\section{Conclusion}
In this study, we elucidated the representation power of \ARFINNs{} by proving their \(L^p\)-universality or \(\sup\)-universality for \(\CtwoDomainDiff\).
Along the course, we invoked a structure theorem from differential geometry to establish an equivalence of the universalities for \(\CtwoDomainDiff\), \(\OneDimTriangular\), and \(\Triangular\), which itself is of theoretical interest.
Our result advances the theoretical understanding of \ARFINNs{} by formally showing that most of the \ARFINN{} architectures already yield \(L^p\)-universal approximators and that the different flow layer designs purely contribute to the efficiency of approximation, not much to the capacity of the model class.
Comparing the approximation efficiency of different layer designs is an important area in future work.
Also, the approximation efficiency for a better-behaved subset of \(\CtwoDomainDiff\) (e.g., bi-Lipschitz ones) remains as an open question for future research.

\begin{ack}
\acknowledgmentContent{}
\end{ack}

\printbibliography
 \clearpage

\begin{appendices}
\global\csname @topnum\endcsname 0
\global\csname @botnum\endcsname 0

This is the Supplementary~Material for ``Coupling-based Invertible Neural Networks Are Universal Diffeomorphism Approximators.''
Table~\ref{tbl:notation-table} summarizes the abbreviations and the symbols used in the paper.
Figure~\ref{fig:proof-flow-chart} depicts the relations among the notions of universalities appearing in this paper and how they are connected by the sections in this Supplementary~Material.

\begin{table}[tbph]
  \caption{Abbreviation and notation table}
  \label{tbl:notation-table}
  \centering
  \begin{tabular}{ll}
    \toprule
    Abbreviation/Notation & Meaning \\
    \midrule
    CF-INN & Invertible neural networks based on coupling flow\\
    IAF & Inverse autoregressive flow\\
    DSF & Deep sigmoidal flow\\
    SoS & Sum-of-squares polynomial flow\\
    MLP & Multi-layer perceptron\\
    \midrule
    CF, $\AutoRegressive{k}{\tau}{\theta}$ & Coupling flow \\
    ACF, $\Aff{k}{s}{t}$ & Affine coupling flow\\
    $\ACFINNUniversalClass$ & Set of functions from $\mathbb{R}^{d-1}$ to $\mathbb{R}$\\
    $\FSACFH, \Aff{d-1}{s}{t}$ & $\ACFINNUniversalClass$-single-coordinate affine coupling flows (\(s, t \in \ACFINNUniversalClass\))\\
    $\FLin$ & Set of invertible affine transformations\\
    $\FGL$ & Set of invertible linear transformations\\
    \midrule
    $\ARFINNFlow$ & Generic notation for a set of invertible functions\\
    $\INN{\ARFINNFlow}$ & Set of invertible neural networks based on \(\ARFINNFlow\)\\
    \midrule
    $\CtwoDomainDiff$ & Set of all $C^2$-diffeomorphisms with \(C^2\)-diffeomorphic domains\\
    $\Triangular$ & Set of all 
    $C^\infty$-increasing triangular mappings\\
    \(\CrOneDimTriangular\) &Set of all $C^r$-single-coordinate transformations\\
    $\DcRD$ & Group of compactly-supported $C^2$-diffeomorphisms (on \(\ReD\))\\
    \midrule
    $\Euclideannorm{\cdot}$ & Euclidean norm\\
    $\opnorm{\cdot}$ & Operator norm\\
    $\LpKnorm{\cdot}$ & $L^p$-norm ($p\in [1,\infty)$) on a subset $K\subset \mathbb{R}^d$\\
    $\inftyKnorm{\cdot}$ & Supremum norm on a subset $K\subset \mathbb{R}^d$\\
    \(\Indicator{A}(\cdot)\) & Indicator (characteristic) function of \(A\) \\
    \bottomrule
  \end{tabular}
\end{table}

\newcommand{\flowchartUniversal}[2]{\(#2\)-univ. for \(#1\)}

\newcommand{\SpecialCase}{S.C.}
\newcommand{\flowchartLpUniversal}[1]{\flowchartUniversal{#1}{L^p}}
\newcommand{\flowchartSupUniversal}[1]{\flowchartUniversal{#1}{\mathrm{sup}}}
\newcommand{\rightnode}[4]{\node[#1, right=\flowchartBaseDistance of #2](#3){#4}}
\newcommand{\leftnode}[4]{\node[#1, left=\flowchartBaseDistance of #2](#3){#4}}

\newcommand{\Mytextboxwidth}{1.85cm}
\def \flowchartBaseDistance {0.85}
\newcommand{\Myarrowshift}{1.5ex}
\newcommand{\MyarrowTextShift}{1.2ex}
\newcommand{\arrowtextwidth}{2.5cm}

\newcommand{\rarrow}[3]{\draw[transform canvas={yshift=\Myarrowshift},->] (#1) -- node[text width=\arrowtextwidth, above, align=center, yshift=\MyarrowTextShift]{#3} (#2)}
\newcommand{\curvedrarrow}[3]{\draw[bend left, transform canvas={yshift=\Myarrowshift},->] (#1) to node[
midway, fill=white,
]{#3} (#2);}
\newcommand{\larrow}[3]{\draw[transform canvas={yshift=-\Myarrowshift},->] (#1) -- node[text width=\arrowtextwidth, below, align=center, yshift=-\MyarrowTextShift]{#3} (#2)}
\newcommand{\architecturearrow}[3]{\draw[transform canvas={xshift=0ex},->] (#1) -- node[text width=\arrowtextwidth]{#3} (#2)}

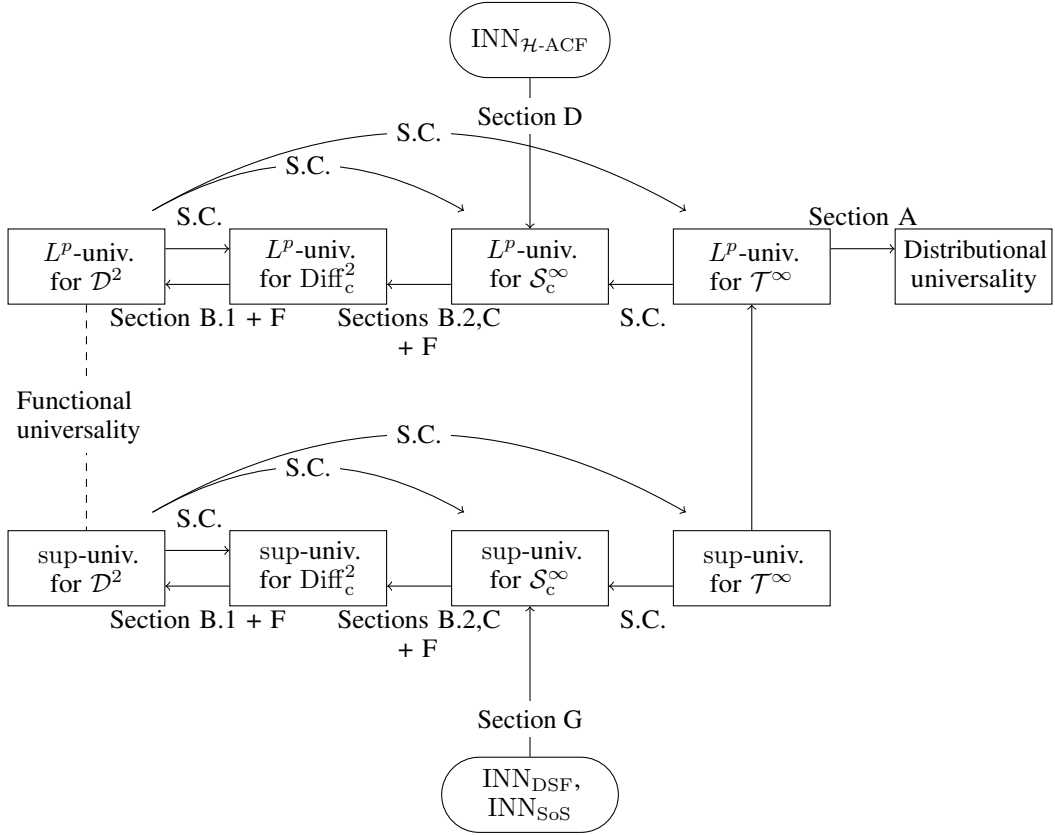
\begin{figure}
\centering
\begin{tikzpicture}
\tikzset{Prop/.style={rectangle,  draw,  text centered, text width=\Mytextboxwidth, minimum height=1cm}};
\tikzset{Architecture/.style={rounded rectangle, draw, text centered, text width=\Mytextboxwidth, minimum height=1cm}};

\node[Prop](c2lp)at (0,0){\flowchartLpUniversal{\CtwoDomainDiff}};
\rightnode{Prop}{c2lp}{cptlp}{\flowchartLpUniversal{\cptDomainDiff}};
\rightnode{Prop}{cptlp}{t1lp}{\flowchartLpUniversal{\OneDimTriangular}};
\rightnode{Prop}{t1lp}{tlp}{\flowchartLpUniversal{\Triangular}};
\rightnode{Prop}{tlp}{weak}{Distributional universality};
\node[Prop, below=3 of tlp](tsup){\flowchartSupUniversal{\Triangular}};
\leftnode{Prop}{tsup}{t1sup}{\flowchartSupUniversal{\OneDimTriangular}};
\leftnode{Prop}{t1sup}{cptsup}{\flowchartSupUniversal{\cptDomainDiff}};
\leftnode{Prop}{cptsup}{c2sup}{\flowchartSupUniversal{\CtwoDomainDiff}};

\begin{scope}[every node/.style={font= \normalsize}]
\draw[dashed] (c2lp) -- (c2sup) node[midway, fill=white, text width=\Mytextboxwidth]{Functional universality};
\rarrow{c2lp}{cptlp}{\SpecialCase{}};
\larrow{cptlp}{c2lp}{Section~\ref{sec:appendix:reduction to cpt supp}
+ \ref{appendix: compatibility of approximation and composition}
};
\curvedrarrow{c2lp}{t1lp}{\SpecialCase{}};
\curvedrarrow{c2lp}{tlp}{\SpecialCase{}};
\larrow{t1lp}{cptlp}{Sections~\ref{sec:appendix:Dc2 to Sinfty},\ref{sec:appendix:Dc2 to nearly-Id}
+ \ref{appendix: compatibility of approximation and composition}
};
\larrow{tlp}{t1lp}{\SpecialCase{}};
\rarrow{tlp}{weak}{Section~\ref{appendix: from lp to dist}
};
\draw[transform canvas={xshift=0ex},->] (tsup) -- (tlp);
\larrow{tsup}{t1sup}{\SpecialCase{}};
\larrow{t1sup}{cptsup}{Sections~\ref{sec:appendix:Dc2 to Sinfty},\ref{sec:appendix:Dc2 to nearly-Id}
+ \ref{appendix: compatibility of approximation and composition}
};
\rarrow{c2sup}{cptsup}{\SpecialCase{}};
\curvedrarrow{c2sup}{t1sup}{\SpecialCase{}};
\curvedrarrow{c2sup}{tsup}{\SpecialCase{}};
\larrow{cptsup}{c2sup}{Section~\ref{sec:appendix:reduction to cpt supp}
+ \ref{appendix: compatibility of approximation and composition}
};
\end{scope}
\node[Architecture, below=2 of t1sup
](NAF){\(\INN{\FDSF}\), \(\INN{\FSoS}\)};
\node[Architecture, above=2 of t1lp, xshift=0ex](ACFINN){\(\INN{\FSACFH}\)};
\draw[transform canvas={xshift=0ex},->] (NAF) -- (t1sup) node[fill=white, near start]{Section~\ref{appendix:sec:examples}};
\draw[transform canvas={xshift=0ex},->] (ACFINN) -- (t1lp) node[fill=white, near start]{Section~\ref{appendix: theorem 2 proof}
};

\end{tikzpicture}
\caption{
Informal diagram of the relations among propositions and lemmas connecting them. Here, \(p \in [1, \infty)\).
\emph{S.C.} stands for ``special case'' and indicates that the notion of universality implies the other as a special case.
\emph{DSF} stands for \emph{deep sigmoidal flow}, and \emph{SoS} stands for \emph{sum-of-squares polynomial flow}.
}
\label{fig:proof-flow-chart}
\end{figure}

\section{Proof of Lemma~\ref{lem:body:distributional-universality}: From \(L^p\)-universality to distributional universality}\label{appendix: from lp to dist}

Here, we prove Lemma~\ref{lem:appendix:lp to dist}, which corresponds to Lemma~\ref{lem:body:distributional-universality} in the main text.

First, note that the larger \(p\), the stronger the notion of \(L^p\)-universality: if a model \(\INNModelGeneric\) is an \(L^p\)-universal approximator for \(\mathcal{F}\), it is also an $L^q$-universal approximator for \(\mathcal{F}\) for all $1 \leq q \leq p$.
In particular, we use this fact with \(q = 1\) in the following proof.
\begin{lemma}[Lemma~\ref{lem:body:distributional-universality} in the main text]\label{lem:appendix:lp to dist}
Let \(p \in [1, \infty)\).
Suppose \(\INNModelGeneric\) is an $L^p$-universal approximator for $\Triangular$. Then \(\INNModelGeneric\) is a distributional universal approximator.
\end{lemma}
\begin{proof}
We denote by ${\rm BL}_1$ the set of bounded Lipschitz functions $f\colon{}\R^d\rightarrow \R$ satisfying  $\Vert f\Vert_{\sup, \R^d}+L_f\le 1$, where $L_f$ denotes the Lipschitz constant of $f$.
Let $\mu, \nu$ be absolutely continuous probability measures, and 
take any $\varepsilon>0$. 
By Theorem~11.3.3 in \cite{DudleyReal2002}, it suffices to show that there exists $g\in\INNModelGeneric$ such that
\[\beta(g_*\mu, \nu):=\sup_{f\in {\rm BL}_1}\left|\int_{\R^d}f\,dg_*\mu-f\,d\nu\right|<\varepsilon.\]
Let $p,q\in L^1(\R^d)$ be the density functions of $\mu$ and $\nu$ respectively.  Let $\phi\in L^1(\R^d)$ be a positive \(C^\infty\)-function such that $\int_{\R^d} \phi(x)dx=1$ (for example, Gaussian distribution), and for $t>0$, put $\phi_t(x):=t^{-d}\phi(x/t)$.  We define  $\mu_t:=\phi_t*pdx$ and $\nu_t:=\phi_t*qdx$. Since both $\Vert \phi_{t}*p-p\Vert_{1,\R^d}$ and $\Vert \phi_{t}*q-q\Vert_{1,\R^d}$ converges to 0 as $t\rightarrow0$, there exists $t_0>0$ such that for any continuous mapping $G:\R^d\rightarrow \R^d$, 
\begin{align*}
\left|\int_{\R^d}f\,dG_*\mu_{t_0}-f\,dG_*\mu\right|&<\frac{\supRangenorm{\ReD}{f}\varepsilon}{5},\quad
\left|\int_{\R^d}f\,d\nu_{t_0}-f\,d\nu\right|<\frac{\supRangenorm{\ReD}{f}\varepsilon}{5}.
\end{align*}
By using Lemma \ref{existence of transformation for probability measures} below, there exists $T\in\Triangular$ such that $T_*\mu_{t_0}=\nu_{t_0}$.  Let $K\subset \R^d$ be a compact subset such that
\[1-\mu_{t_0}(K)<\frac{\varepsilon}{5}.\]
By the assumption, there exists  $g\in\INNModelGeneric$  such that 
\[\int_K |T(x)-g(x)|dx<\frac{\varepsilon}{5\sup_{x\in K}|\phi_{t_0}*p(x)|}.\]
Thus for any $f\in{\rm BL}_1$, we have
\begin{align*}
&\left|\int_{\R^d}f\,dg_*\mu-f\,d\nu\right|\\
&\le \left|\int_{\R^d}f\,dg_*\mu_{t_0}-f\,dg_*\mu\right| +
\left|\int_{\R^d}f\,d\nu_{t_0}-f\,d\nu\right| \\
&\hspace{12pt}+\left|\int_{\R^d\setminus K}f\circ T\,d\mu_{t_0}\right|
+\left|\int_{\R^d\setminus K}f\circ g\,d\mu_{t_0}\right|
+\int_{K}\left|f(T(x))-f(g(x))\right|\,d\mu_{t_0}(x)\\
&<\frac{\supRangenorm{\ReD}{f}\varepsilon}{5}+\frac{\supRangenorm{\ReD}{f}}{5}+\frac{\supRangenorm{\ReD}{f}\varepsilon}{5}+\frac{\supRangenorm{\ReD}{f}\varepsilon}{5}+\frac{L_f\varepsilon}{5}\\
&\le\varepsilon,
\end{align*}
where $L_f$ is the Lipschitz constant of $f$. 
Here we used $\supRangenorm{\Re^d}{f}+L_f\leq 1$. 
Therefore, we have $\beta(g_*\mu,\nu)<\varepsilon$.
\end{proof}

The following lemma is essentially due to \cite{HyvarinenNonlinear1999}.
\begin{lemma}
\label{existence of transformation for probability measures}
Let $\mu$ be a probability measure on $\R^d$ with a \(C^\infty\) density function $p$. Let $U:=\{x\in \R^d : p(x)>0\}$. Then there exists a diffeomorphism $T:U\rightarrow (0,1)^d$ such that its Jacobian is upper triangular matrix with positive diagonal, and $T_*\mu={\rm U}(0,1)^d$. Here, ${\rm U}(0,1)^d$ is the uniform distribution on $[0,1]^d$.
\end{lemma}
\begin{proof}
Let $q_i(x_1,\dots,x_i):=\int_{\R^{d-i}}p(x_1,\dots,x_{i+1},\dots,x_d)\,dx_{i+1}\dots dx_d$. Then we define $T:U\rightarrow (0,1)^d$ by
\[T(x_1,\dots,x_d):=\left(\int_{-\infty}^{x_i}\frac{q_i(x_1,\dots,x_{i-1},y)}{q_{i-1}(x_1,\dots,x_{i-1})}dy\right)_i.\]
Then we see that $T$ is a diffeomorphism and its Jacobian is upper triangular with positive diagonal elements. Moreover, by a direct computation, we have $T_*d\mu=U(0,1)$.
\end{proof}

We include a proof for the statement that that any probability measure on $\Re^m$ is arbitrarily approximated by an absolutely continuous probability measure in the weak convergence topology:
\begin{lemma}
\label{lem:appendix: abs aprox any}
Let $\mu$ be an arbitrary probability measure of $\Re^m$.
Then there exists a sequence $\{\mu_n\}_{n=1}^\infty$ such that $\mu_n$ weakly converges to $\mu$.
\end{lemma}
\begin{proof}
Let $\phi$ be a positive $C^\infty$ function such that $\int_{\Re^m} \phi(x) dx=1$.
For $t>0$, put $\phi_t(x):=t^{-m}\phi(x/t)$.
We define
\begin{align*}
   w_t(x)=\int_{\Re^m} \phi_t(x-y)d\mu(y).
\end{align*}
We prove the absolutely continuous measure $w_tdx$ weakly converges to $\mu$ as $t\rightarrow 0$.
In fact, for any bounded continuous function $f$, we have
\begin{align*}
    \left|\int_{\Re^m} fw_t dx - \int fd\mu \right| 
    &= \left| \int \int_{\Re^m} \left(f(y+tx)-f(y)\right)\phi(x) dxd\mu(y) \right|\\
    &\le \int \int_{\Re^m} |f(y+tx)-f(y)|\phi(x)dx d\mu(y).
\end{align*}
Since $f$ is bounded and $\phi$ is absolutely integrable, by the dominated convergence theorem, as $t\rightarrow 0$, we have
\[\int_{\Re^m} fw_t dx \rightarrow  \int fd\mu,\]
namely, $w_tdx$ weakly converges to $\mu$.
\end{proof} 
\section{Proof of Theorem~\ref{thm:body:diffeo-universal-equivalences}: Equivalence of universality properties}
\label{sec:appendix:universality-proof}
In this section, we provide the proof details of Theorem~\ref{theorem:main:1} in the main text.
Section~\ref{sec:appendix:reduction to cpt supp} explains the reduction from \(\CtwoDomainDiff\) to \(\DcRD\), and
Section~\ref{sec:appendix:Dc2 to Sinfty} explains the reduction from \(\DcRD\) to \(\CinftyOneDimTriangular\) and permutations of variables.

Here, we formally repost the proof of Theorem~\ref{theorem:main:1} which has been essentially completed in Section~\ref{subsection: outline of the proof of theorem 1}.
\begin{proof}[Proof of Theorem~\ref{theorem:main:1}]
Since we have $\CinftyOneDimTriangular\subset \Triangular\subset \CtwoDomainDiff$, it is sufficient to prove that the universal approximation properties for $S^\infty$ imply those for $\CtwoDomainDiff$.
Therefore, we focus on describing the reduction from \(\CtwoDomainDiff\) to \(\CinftyOneDimTriangular\).
First, by combining Lemma~\ref{smoothing lemma} with the \(L^p\)-universality (in the case \ref{main thm: A}) or the \(\sup\)-universality (in the case \ref{main thm: B}) of $\INN{\ARFINNFlow}$ for \(\CinftyOneDimTriangular\), we obtain the \(L^p\)-universal (resp. \(\sup\)-universal) approximation property for $\CtwoOneDimTriangular$.
Now, in light of Lemma~\ref{red to comp. supp. diff} and Theorem~\ref{thm: singles and permutations}, we obtain the assertion of Theorem~\ref{thm: equivalence S2-D2} in the main text, i.e., for any $f\in\CtwoDomainDiff$ and compact subset $K\subset U_f$, there exist $W_1,\dots, W_r \in \FLin$ and $\tau_1,\dots,\tau_r\in \CtwoOneDimTriangular$  and $b\in\ReD$ such that $f(x)=W_1\circ\tau_1\circ\cdots\circ W_r\circ\tau_r(x)$ for all $x\in K$.
Given this decomposition, we combine the \(L^p\)-universality (in the case \ref{main thm: A}) or the \(\sup\)-universality (in the case \ref{main thm: B}) of $\INN{\ARFINNFlow}$ for $\CtwoOneDimTriangular$ with Proposition~\ref{prop: compatibility of approximation} to obtain the assertion of Theorem~\ref{theorem:main:1}.
\end{proof}

\subsection{From \(\CtwoDomainDiff\) to \(\DcRD\)
}
\label{sec:appendix:reduction to cpt supp}
In this section, we describe how the approximation of $\CtwoDomainDiff$ is reduced to that of $\DcRD$ when we are only concerned with its approximation on a compact set.

\begin{lemma}
\label{lem1}
\label{red to comp. supp. diff}
Let $f \colon U \to \ReD$ be an element of $\CtwoDomainDiff$, and let $K\subset U$ be a compact set. Then, there exists $h \in \DcRD$ and an affine transform $W \in \FLin$ such that \[\Restrict{W\circ h}{K}=\Restrict{f}{K}.\]
\end{lemma}
\begin{proof}[Proof of Lemma~\ref{lem1}]
We denote the injections of $U$ and $f(U)$ into $\mathbb{R}^d$ by $\iota_1\colon  U\hookrightarrow \mathbb{R}^d$ and $\iota_2\colon  f(U) \hookrightarrow \mathbb{R}^d$, respectively.
Since $U$ is $C^2$-diffeomorphic to $\mathbb{R}^d$ and $f$ is $C^2$-diffeomorphic, $f(U)$ is also $C^2$-diffeomorphic to $\mathbb{R}^d$.
By applying Theorem 3.3 in \cite{BernardExpressing2018} to $\iota_1 \circ \Restrict{f^{-1}}{f(U)}\colon  f(U) \to \mathbb{R}^d$ and the injection $\iota_2$, we can obtain diffeomorphisms $F_1\colon f(U)\rightarrow \mathbb{R}^d$ and $F_2\colon  f(U) \rightarrow \mathbb{R}^d$ such that $\Restrict{F_1}{f(K)} = \Restrict{f^{-1}}{f(K)}$ and $\Restrict{F_2}{f(K)} = \Identity_{f(K)}$, where $\Identity_{f(K)}$ denotes the identity map on ${f(K)}$.
Let $F:=F_2\circ F_1^{-1}\colon  \mathbb{R}^d \to \mathbb{R}^d$.
By definition, we have $\Restrict{F}{K} = \Restrict{f}{K}$.

Take a sufficiently large open ball $B$ centered at 0 such that $K\subset B$. 
Let $W\in\FLin{}$ such that $W(x)=\Jac{F}^{-1}(0)(x-F(0))$. Then by Lemma \ref{extension lemma for function on ball} below, 
we conclude that there exists a compactly supported diffeomorphism $h\colon \mathbb{R}^d\rightarrow\mathbb{R}^d$ such that $\Restrict{W\circ h }{K}=\Restrict{F}{K} = \Restrict{f}{K}$.
\end{proof}

\begin{lemma}
\label{extension lemma for function on ball}
Let $B_r\subset\ReD$ be an open ball of radius $r$ with origin $0$, and let $f:B_r\rightarrow f(B_r)\subset \ReD$ be a $C^2$-diffeomorphism onto its image such that $f(0)=0$ and $\Jac{f}(0)=I$.
Let $\varepsilon\in (0,r/2)$.                                      Then there exists $h\in \DcRD$ such that $f(x)=h(x)$ for any $x\in B_{r-\varepsilon}$.
\end{lemma}
\begin{proof}
Put $\delta:=\varepsilon/(2r-\varepsilon)$, and define $I_\delta:=(-1-\delta,1+\delta)$.  We define $F:B_{r-\varepsilon/2}\times I_\delta \rightarrow\ReD$ by
\[F(x,t):=
\begin{cases}
\frac{f(tx)}{t} & \text{ if }t\neq 0,\\
x & \text{ if } t=0.
\end{cases}
\]
Let $U:=F(B_{r-\varepsilon/2})$ and let $F^\dagger:U\times I_\delta\rightarrow B_{r-\varepsilon/2}$ such that $F^\dagger(F(x,t))=x$ for any $(x,t)\in U$.  Fix a compactly supported function on $\ReD\times I_\delta$ such that for $(x,t)\in F\big(\overline{B_{r-\varepsilon}}\times [-1,1] \big)$, $\phi(x,t)=1$, and for $(x,t)\notin U$ $\phi=0$. Then we define $H:\ReD\times I_\delta\rightarrow \ReD$ by
\[H(x,t):=\phi(x,t)\frac{\partial F}{\partial t}(F^\dagger(x,t),t).\]
Since $f$ is $C^2$ diffeomorphism, there exists $L>0$ such that for any $t\in I_\delta$, $\Vert H(x,t)-H(y,t)\Vert < L\| x- y\|$ with $x,y\in \ReD$. Thus the differential equation
\[\frac{dz}{dt}=H(z,t),~~z(0)=x\]
has a unique solution $\phi_x(t)$. Then $h(x):=\phi_x(1)$ is the desired extension.
\end{proof}
Here, we remark that Lemma~\ref{extension lemma for function on ball} is a modified version of Lemma~D.1 in \citet{BernardExpressing2018}, with a correction to make it explicit that the extended diffeomorphism is compactly supported. Their Lemma~D.1 does not explicitly state that it is compactly supported, but by Theorem~1.4 in Section~8 of \citet{HirschDifferential1976}, it can be shown that the diffeomorphism is actually compactly supported.

\subsection{From \(\DcRD\) to \(\CinftyOneDimTriangular\) and permutations}
\label{sec:appendix:Dc2 to Sinfty}
The goal of this section is to show Theorem~\ref{thm: singles and permutations}, which reduces the approximation problem of $\DcRD$ to that of $\CtwoOneDimTriangular$, and Lemma~\ref{smoothing lemma}, which reduces from $\CtwoOneDimTriangular$ to $\CinftyOneDimTriangular$.

\begin{theorem}
\label{thm: singles and permutations}
Let $f\in \DcRD$. Then there exist $\tau_1,\dots,\tau_n \in \CtwoOneDimTriangular \cap \DcRD$, and permutations of variables $\sigma_1,\dots,\sigma_n\in \mathfrak{S}_d$, such that
\[f=\tau_1\circ \sigma_1\circ\dots\circ \tau_n\circ\sigma_n.\]
\end{theorem}
\begin{proof}
Combining
Corollary~\ref{corollary: from dcrd to nearlyId},
Lemma~\ref{lem: operator norm to principal minors},
and Lemma~\ref{lemma:red_to_one_dim},
we have the assertion.
\end{proof}

We defer the statement and proof of Corollary~\ref{corollary: from dcrd to nearlyId}, which describes the key properties of \(\DcRD\), to Section~\ref{sec:appendix:Dc2 to nearly-Id}.
In the remainder of this section, we describe Lemma~\ref{lem: operator norm to principal minors}, Lemma~\ref{lemma:red_to_one_dim}, and Lemma~\ref{smoothing lemma}.
First, Lemma~\ref{lem: operator norm to principal minors} claims that the nearly-$\Identity$ elements necessarily satisfy the condition of Lemma~\ref{lemma:red_to_one_dim} below.
\begin{lemma}
\label{lem: operator norm to principal minors}
Let $A=(a_{i,j})_{i,j=1,\dots,d}$ be a matrix. If $\Vert A- I_d \Vert_{\rm op}<1$, then for $k=1,\dots,d$, the $k$-th trailing principal submatrix $A_k:=(a_{i+k-1,j+k-1})_{i,j=1,\dots,d-(k-1)}$ of $A$ is invertible. Here $I_d$ is a unit matrix of degree $d$.
\end{lemma}
\begin{proof}
Let $v\in\Re^{d-k+1}$ with $\Vert v\Vert=1$, and put $w:=(0,\dots,0,v)\in\ReD$. Then we have $1>\Vert (A-I_d)w\Vert^2\ge\Vert (A_k-I_k)v\Vert^2$. Thus $\Vert A_k-I_k\Vert<1$.  Since $\sum_{r=0}^\infty(I_k-A_k)^r$ absolutely converges, and it is identical to the inverse of $A_k$, we have that $A_k$ is invertible.
\end{proof}
We apply the following lemma together with Lemma~\ref{lem: operator norm to principal minors} to decompose nearly-$\Identity$ elements into $\CtwoOneDimTriangular$ and permutations.
For $a\in \Na$, we denote the set of $a$-by-$a$ real-valued matrices by $M(a, \mathbb{R})$.
\begin{lemma}\label{lemma:red_to_one_dim}
\label{lem6}
Let $r$ be a positive integer and 
$f\colon \mathbb{R}^d\rightarrow\mathbb{R}^d$ a compactly supported $C^r$-diffeomorphism. We write $f=(f_1,\dots,f_d)$ with $f_i\colon \mathbb{R}^d\rightarrow \mathbb{R}$.  
For $k\in [d]$, let $\Delta^f_k(\bm{x})\in M(d-(k-1), \R)$ be the $k$-th trailing principal submatrix of Jacobian matrix of $f$, whose  $(i,j)$ component is given by $\left(\frac{\partial f_{i+k-1}}{\partial x_{j+k-1}}(\bm{x})\right)$ $(i, j=1,\cdots, d-(k-1))$. 
We assume 
\[\det \Delta^f_k(x)\neq 0 \text{ for any }k\in [d]\text{ and }x\in \R^d. \]
Then there exist compactly supported $C^r$-diffeomorphisms $F_1,\dots,F_d:\mathbb{R}^d\rightarrow \mathbb{R}^d$ in the forms of 
\[F_i(\bm{x}):=(x_1,\dots,x_{i-1},h_i(\bm{x}),x_{i+1},\dots,x_d)\]
for some $h_i\colon \mathbb{R}^d\rightarrow\mathbb{R}$ such that the identity holds:
\[f=F_1\circ\dots\circ F_d. \]
\end{lemma}
\begin{proof}
The proof is based on induction. Suppose that $f$ is in the form of 
\(f(\bm{x})=(f_1(\bm{x}),\dots,f_m(\bm{x}),x_{m+1},\dots,x_d).\)
By means of induction with respect to $m$, we prove that there exist compactly supported $C^r$-diffeomorphisms $F_1,\dots,F_m:\mathbb{R}^d\rightarrow \mathbb{R}^d$ in the forms of 
\(F_i(\bm{x}):=(x_1,\dots,x_{i-1},h_i(\bm{x}),x_{i+1},\dots,x_d)\)
for some $h_i:\mathbb{R}^d\rightarrow\mathbb{R}$ such that
\(f=F_1\circ\dots\circ F_m\).

In the case of $m=1$, the above is clear.  Assume that the statement is true in the case of any $k<m$.
Define
\begin{align*}
    F(x_1,\dots,x_d)&:=(x_1,\dots,x_{m-1},f_m(\bm{x}),x_{m+1},\dots,x_d), \\
    \tilde{f}&:=f\circ F^{-1}. 
\end{align*}
Note that 
$F$ is a compactly supported $C^r$-diffeomorphism from $\R^d$ to $\R^d$. 
In fact, compactly supportedness and surjectivity of $F$ comes from the compactly supportedness of $f$.
Moreover, since we have $\det \Jac F_x=\frac{\partial f_m}{\partial x_m}(x)\neq 0$ for any $x\in \R^d$ by the assumption on $f$, $F$ is injective and is a $C^r$-diffeomorphism from $\R^d$ to $\R^d$ by inverse function theorem. 
Therefore, $\tilde{f}$ is also a $C^r$-diffeomorphism from $\R^d$ to $\R^d$.
We show 
that $\tilde{f}$ is of the form $\tilde{f}(\bm{x})=(g_1(\bm{x}), \cdots, g_{m-1}(\bm{x}), x_m,\cdots, x_d)$ for some $C^r$-functions $g_i\colon \R^d\to \R$ $(i=1,\cdots, m-1)$ satisfying $\det \Delta^{\tilde{f}}_k(x)\neq 0$ for any $x\in \R^d$ and $k\in [d]$. 
From Lemma~\ref{lemma:inverse_component}, 
there exist $g_i, h\in C^r(\R^d)$ $(i=1,\cdots, m)$ such that 
\begin{align*}
    f^{-1}(\bm{x})&=(g_1(\bm{x}), \cdots, g_m(\bm{x}), x_{m+1}, \cdots, x_d)\\
    F^{-1}(\bm{x})&=(x_1,\cdots, x_{m-1}, h(\bm{x}), x_{m+1}, \cdots, x_d). 
\end{align*}
Then we have 
\begin{align*}
\tilde{f}^{-1}(\bm{x})=F\circ f^{-1}(\bm{x})
&= (g_1(\bm{x}),\cdots, g_{m-1}(\bm{x}), f_m(f^{-1}(\bm{x})), x_{m+1},\cdots, x_{d})\\
&=(g_1(\bm{x}), \cdots,g_{m-1}(\bm{x}), x_m, \cdots, x_d). 
\end{align*}
Therefore, from Lemma~\ref{lemma:inverse_component}, 
$\tilde{f}$ is of the following form 
\[ \tilde{f}(x)= f\circ F^{-1}(x)= (f_1\circ F^{-1}(x), \cdots, f_{m-1}\circ F^{-1}(x), x_m,\cdots, x_d).  \]
Moreover, by the form of $F^{-1}$ and $f$, 
we have $\Jac{\tilde{f}}(x)=\Jac{f}(F^{-1}(x))\circ \Jac{F^{-1}}(x)$ and 
\[  \Jac f=
\begin{pmatrix}
A & \\
  & I
\end{pmatrix}, \quad
\Jac(F^{-1})=
\begin{pmatrix}
I_{m-1} & &   \\
  \frac{\partial h}{\partial x_1} & \cdots & \frac{\partial h}{\partial x_d}\\
 & & I_{d-m} 
\end{pmatrix}
\]
for some $A\in M(m,\R)$ with all the trailing principal minors nonzero. 
Therefore, we obtain $\det \Delta^f_k(x)\neq 0$ for any $x\in \R^d$ and $k\in [d]$. 
Here, by the assumption of the induction, 
there exist compactly supported $C^r$-diffeomorphisms $F_i\colon \R^d\to \R^d$ and $h_i\in C^r(\R^d)$ $(i=1,\cdots, m-1)$ such that
\[\tilde{f}=F_1\circ \cdots \circ F_{m-1},\ F_i(\bm{x})=(x_1,\cdots x_{i-1}, h_i(x), x_{i+1}, \cdots, x_d). \]
Thus $f= \tilde{f}\circ F$ has a desired form.
\end{proof}

\begin{lemma}\label{lemma:inverse_component}
\label{lem5}
Let $r$ be a positive integer and $f\colon \R^d\to \R^d$ $C^r$-diffeomorphism of the form 
\[ f(\bm{x}):=(f_1(\bm{x}), \cdots, f_m(\bm{x}), x_{m+1}, \cdots, x_d), \]
where $f_i\colon \R^d\rightarrow\R$ belongs to $C^r (\R^d)$ $(i=1,\cdots, m)$. 
Then the inverse map $f^{-1}$ becomes of the form
\[ f^{-1}(\bm{x})= (g_1(\bm{x}), \cdots, g_m(\bm{x}), x_{m+1},\cdots x_d), \]
where $g_i:\R^d\rightarrow\R$ belongs to $C^r(\R^d)$ for $i=1,\cdots, m$. 
\end{lemma}
\begin{proof}
We write $f^{-1}(\bm{x})=(h_1(\bm{x}), \cdots, h_d(\bm{x}))$, where $h_i\in C^r(\R^d)$ $(i=1,\cdots, d)$. 
Then by the definition of the inverse map, the identity 
\[ (x_1,\cdots, x_d)=f\circ f^{-1}(\bm{x}) =(f_1(h_1(\bm{x})),\cdots, f_m(h_m(\bm{x})), h_{m+1}(\bm{x}),\cdots, h_d(\bm{x})) \]
holds for any $\bm{x}\in \R^d$, which implies that we obtain $h_i(x)=x_i$ $(i=m+1,\cdots, d)$. This completes the proof of the lemma. 
\end{proof}

The following Lemma~\ref{smoothing lemma} is used in the main text in reducing the approximation problem from \(\CtwoOneDimTriangular\) to \(\CinftyOneDimTriangular\).
We say that $f\colon \R^d\to \R$ is a \emph{locally $L^p$-function} if 
$\int_K|f(x)|^pdx<\infty$ holds for any compact set $K\subset \R^d$.
\begin{definition}[Last-increasing]
We say that a map \(f: \ReD \to \Re\) is \emph{last-increasing} if, for any $(a_1,\dots,a_{d-1})\in\R^{d-1}$, the function \(f(a_1, \ldots, a_{d-1}, x)\) is strictly increasing with respect to $x$.
\end{definition}
\begin{lemma}
\label{smoothing lemma}
Let $\tau\colon \R^d\rightarrow \R$ be a last-increasing locally $L^p$-function. Then for any compact subset $K\subset \R^d$ and any $\varepsilon>0$, there exists a last-increasing \(C^\infty\)-function $\tilde{\tau}\colon\mathbb{R}^d\rightarrow \R$ satisfying
 \[ \|\tau-\tilde{\tau}\|_{p,K}<\varepsilon. \]
Moreover, if $\tau$ is continuous, there exists a last-increasing \(C^\infty\)-function $\tilde{\tau}$ such that
\[\supKnorm{\tau-\tilde{\tau}}<\varepsilon.\]
\end{lemma}
\begin{proof}
Let $\phi:\R^d\rightarrow\R$ be a compactly supported non-negative \(C^\infty\)-function with $\int |\phi(x)|dx =1$ such that for any $(a_1,\dots,a_{d-1})\in\R^{d-1}$, the function $\phi(a_1,\dots,a_{d-1}, x)$ of $x$ is even and decreasing on $\{x>0 : \phi(a_1,\dots,a_{d-1},x)>0\}$.
For $t>0$, we define $\phi_t(x):=t^{-d}\phi(x/t)$. Then we see that $\tau_t:=\phi_t*\tau$ is a \(C^\infty\)-function. 
We take any $\bm{a}\in \R^{d-1}$. 
We verify that $\tau_t(\bm{a}, x_d)$ is strictly increasing with respect to $x_d$.
Take any $x_d, x_d'\in \R$ satisfying $x_d>x_d'$. 
Since $\tau$ is strictly increasing, we have 
\begin{align*}
\tau_t(\bm{a}, x_d)-\tau_t(\bm{a}, x_d')
&=\int_{\R^d} \phi_t(x) (\tau((\bm{a},x_d)-x)-\tau(( \bm{a},x_d')-x))dx>0. 
\end{align*}
Thus for any $(a_1,\dots,a_{d-1})\in\R^{d-1}$, 
the \(C^\infty\)-function $\tau_t(a_1,\dots,a_{d-1},x)$ is strictly increasing for with respect to $x$. 

Next, take any compact subset $K\subset \R^d$. 
We show
$\|\tau_t-\tau\|_{p,K}\to 0$ as $t\to 0$.  
We prove $\tau_t$ converges $\tau$ as $t\rightarrow 0$.  
Take $R>0$ satisfying $K\subset B(R):=\{x\in \R^d : |x|\leq R \}$. 
We assume $0<t<1$. 
Then we have $\phi_t*\tau=\phi_t*(\Indicator{B(R+1)}\tau)$. 
Since we have $\Indicator{B(R+1)}\tau\in L^p(\R^d)$, we obtain 
\begin{align*}
\| \phi_t * \tau-\tau\|_{p,K}
&=\|\phi_t* (\Indicator{B(R+1)}\tau)-\Indicator{B(R+1)}\tau \|_{p,K}\\
&\leq \| \phi_t *(\Indicator{B(R+1)}\tau)-\Indicator{B(R+1)}\tau\|_{p,\R^d}
\to 0 \quad (t\to 0). 
\end{align*}
Here, we used a property of mollifier $\phi_t$ (see Theorem~8.14 in \cite{FollandReal1999} for example). 

Next, we consider the \(\sup\)-approximation when $\tau$ is continuous. 
By direct computation, we have
\begin{align*}
\sup_{y\in K}|\tau_t(y)-\tau(y)|
& \le \sup_{y\in K}\int_{\R^d} |\phi(x)|\cdot|\tau(y-tx)-\tau(y)| dx\\
& \le C\sup_{(x,y)\in {\rm supp}(\phi)\times K} |\tau(y-tx)-\tau(y)|
\to 0\quad (t\to 0). 
\end{align*}
Here $C:=\sup_{x\in \R^d}|\phi(x)|$. 
Thus in both cases above,  
By taking sufficiently small $t$, we obtain the desired \(C^\infty\)-function $\tilde{\tau}=\tau_t$. 
\end{proof}

\section{Key properties of diffeomorphisms on $\R^d$: From \(\DcRD\) to Nearly-\(\Identity\)}
\label{sec:appendix:Dc2 to nearly-Id}
This section explains the reduction of the universality for \(\DcRD\) to Nearly-\(\Identity\) elements.
The reduction involves a structure theorem from the field of differential geometry.
The results of this section are used as a building block for the proofs in Section~\ref{sec:appendix:Dc2 to Sinfty}.

\begin{definition}[Compactly supported diffeomorphism]

The diffeomorphism $f$ on $\mathbb{R}^d$ is {\em compactly supported} if 
there exists a compact subset $K\subset \mathbb{R}^d$ such that for any $x\notin K$, $f(x)=x$. We denote by $\DcRD$ the space of compactly supported \(C^2\)-diffeomorpshisms.
\end{definition}
The set $\DcRD$ constitutes a group whose group operation is the function composition.
Moreover, $\DcRD$ is a topological group with respect to the \emph{Whitney topology} \citep[Proposition~1.7.(9)]{HallerGroups1995}.
Then there is a crucial structure theorem of $\DcRD$ attributed to Herman, Thurston \cite{ThurstonFoliations1974}, Epstein \cite{Epsteinsimplicity1970}, and Mather \cite{MatherCommutators1974, MatherCommutators1975}:
\begin{fact}
\label{thm: simplicity}
The group $\DcRD$ is simple, i.e., any normal subgroup $H \subset \DcRD$ is either $\{\Identity\}$ or $\DcRD$.
\end{fact}
The assertion is proven in \citet{MatherCommutators1975} for the connected component containing \(\Identity\), instead of the entire set of compactly-supported \(C^2\)-diffeomorphisms when the domain space is a general manifold instead of \(\ReD\). In the special case of \(\ReD\), the connected component containing \(\Identity\) is shown to be \(\DcRD\) itself \citep[Example~1.15]{HallerGroups1995}, hence Fact~\ref{thm: simplicity} follows. 
For details, see \citep[Corollary~3.5 and Example~1.15]{HallerGroups1995}.

As a side note, the assertion of Theorem~\ref{thm: simplicity} is proved to hold generally for \(C^r\)-diffeomorphisms only except for \(r = d+1\) \cite{HallerGroups1995}. Nevertheless, this exception does not cause any problem in our proof, because we apply it with \(r = 2\) and \(d \geq 2\). The limitation only means that the structure of \(C^2\)-diffeomorphisms is better understood than that of \(C^{d+1}\)-diffeomorphisms. Also note that this exception does not affect the approximation capability for \(C^{d+1}\)-diffeomorphisms either as they are contained in \(C^2\) where we perform our theoretical analyses.  For the details of mathematical ingredients, see \cite{Banyagastructure1978}.

Here, we provide a precise definition of the \emph{flow endpoints} introduced in Section \ref{subsection: outline of the proof of theorem 1}.
\begin{definition}[Flow endpoints]
\label{def: flow endpoints in appendix}
A \emph{flow endpoint} is an element of $\DcRD$ which can be represented as $\phi(1)$, where $\phi:[0,1]\rightarrow \DcRD$ is a continuous map such that $\phi(0)=\Identity$ and that $\phi$ is additive, namely, $\phi(s)\circ\phi(t)=\phi(s+t)$ for any $s,t\in[0,1]$ with $s+t\in[0,1]$.
\end{definition}

We use Fact~\ref{thm: simplicity} to prove that a compactly supported diffeomorphism can be represented as a composition of flow endpoints in $\DcRD$.
The following lemma is a restatement of Lemma~\ref{lem:body:flow-approximation} in the main text.
\begin{lemma}\label{lem: non-triviality}
Let $S\subset\DcRD$ be the set of all flow endpoints.
Then, $\DcRD$ coincides with the set of finite compositions of elements in $S$ defined by
\[H:=\{ g_1\circ \cdots \circ g_n : n\ge1, g_1,\dots,g_n \in S\}.\]
\begin{proof}
    In view of Fact~\ref{thm: simplicity}, it is enough to show that $H$ forms a subgroup, that it is normal, and that it is non-trivial.

    First, we prove the $H$ consists a subgroup of $\DcRD$.
    By definition, for any $g, h \in H$, it is immediate to show that $g \circ h \in H$.
    We prove that $H$ is closed under inversion. For this, it suffices to show that $S$ is closed under inversion.
    Let $g= \phi(1) \in S$. Consider the map $\varphi: [0, 1] \to \DcRD$ defined by $\varphi(t) := (\phi(t))^{-1}$.
    Since $\DcRD$ is a topological group \citep[Proposition~1.7.(9)]{HallerGroups1995}, $\varphi$ is continuous. Moreover, it is immediate to show that $\varphi$ is additive in the sense of Definition~\ref{def: flow endpoints in appendix}, and that $\varphi(0) = \Identity{}$.
    Thus, $g^{-1}=\varphi(1)$ is an element of $S$.

    Next, we prove $H$ is normal.
    It suffice to show that $S$ is closed under conjugation since the conjugation $g \mapsto hgh^{-1}$ is a group homomorphism on $\DcRD$.
    Let $g=\phi(1)\in S$, where $\phi: [0,1] \rightarrow \DcRD$ is a continuous map associated to $g$. Then, we define a $\Phi: \ReD \times [0,1] \rightarrow \ReD$ by $\Phi(x,t)=\phi(t)(x)$. We call $\Phi$ a flow associated with $g$. We take arbitrary $h\in \DcRD$.
    Then, the function $\Phi': \mathbb{R}^d\times [0, 1]$ defined by $\Phi'(\cdot, s) := h^{-1} \circ \Phi(\cdot, s) \circ h$ is a flow associated with $h^{-1}gh$, which means $h^{-1}gh\in S$, i.e., $S$ is closed under conjugation.
    
    Finally, we show $H$ is nontrivial.
    It suffice to show that $S$ includes a non-identity element.
    Let $\psi: \mathbb{R} \rightarrow {\rm O}(d)$ be a nontrivial homomorphism of Lie groups, where ${\rm O}(d)$ is a orthogonal group of degree $d$.  Such $\psi$ exists, for example,  let $\psi(t):=\exp(tA)$ for some nonzero skew-symmetric matrix $A$, namely, $A^\top = -A$.
    Let $u:[0,\infty) \rightarrow \mathbb{R}$ be a compactly supported $C^\infty$ function such that its support does not include $0$.
    Then, We define $\Phi: \ReD \times [0,1] \rightarrow \ReD$ by $\Phi(x,t):=\psi(u(|x|)t)x$.
    Then, $\Phi$ is the flow associated with $\Phi(\cdot,1) \in S$, that is a non-identity element.

    \end{proof}
\end{lemma}

\begin{definition}[Nearly-$\Identity$ elements]
Let $f \in \DcRD$.
We say $f$ is \emph{nearly-$\Identity$} if, for any $x\in \R^d$, the Jacobian $\Jac{f}$ of $f$ at $x$ satisfies
\begin{align*}
   \opnorm{\Jac{f}(x)-I}<1,
\end{align*}
where $I$ is the unit matrix.
\end{definition}

\begin{corollary}
\label{cor1}
\label{corollary: from dcrd to nearlyId}
For any $f\in \DcRD$, there exist finite elements $g_1,\dots, g_r\in \DcRD$ such that $f=g_r\circ \dots\circ g_1$ and 
$g_i$ is nearly-$\Identity$ for any $i \in [r]$. 
\end{corollary}
\begin{proof}
Let $S$ be the subset of $\DcRD$ as defined above.
Therefore, by Lemma~\ref{lem: non-triviality}, there exist $h_1,\dots, h_m\in S$ such that $f=h_m\circ\dots\circ h_1$.
For $i\in [m]$, let $\phi_i$  be a flow associated with $h_i$.
Since $[0,1]\ni t \mapsto \Phi_i(\cdot, t)\in \DcRD$ is continuous with respect to Whitney topology and $\Phi_i(\cdot, 0)$ is the identity function, we can take a sufficiently large $n$ such that $\tilde{h}_i:=\Phi_i(\cdot, 1/n)$ is nearly-Id. 
By the additive property of $\Phi_i$, we have 
\begin{align*}
\label{expre}
    f
    = h_m \circ \cdots \circ h_1
    =\underbrace{\tilde{h}_m\circ\dots\circ\tilde{h}_m}_{n\ \text{times}}
    \circ \cdots \circ \underbrace{\tilde{h}_1\circ\dots\circ\tilde{h}_1}_{n\ \text{times}},
\end{align*}
which completes the proof of the corollary.
\end{proof}

\section{Proof of Theorem~\ref{prop:body:acfinn-Lp}: \(L^p\)-universality of \ACFHINN{}}
\label{appendix: theorem 2 proof}
In this section, we provide the proof details of Theorem~\ref{theorem:main:2} in the main text.
The correspondence between this section and Section~\ref{subsec: approximation via ACF} in the main text is as follows:
Steps~1, 2, 3 correspond to Section~\ref{sec:appendix:lp ACFINN approx general},
Step~4 corresponds to Section~\ref{sec:appendix:coordinate-wise independent},
and Step~5 is justified by Proposition~\ref{prop: compatibility of approximation} in Section~\ref{appendix: compatibility of approximation and composition}.

\subsection{Approximation of general elements of \(\CzeroOneDimTriangular\)}
\label{sec:appendix:lp ACFINN approx general}
In this section, we prove the following lemma to construct an approximator for an arbitrary element of $\CzeroOneDimTriangular$ (hence for $\CinftyOneDimTriangular$) within \ACFHINN{}.
It is based on Lemma~\ref{lem: univ. approx.} proved in Section~\ref{sec:appendix:coordinate-wise independent}, which corresponds to a special case.

Here, we rephrase Theorem~\ref{theorem:main:2} as in the following:
\begin{lemma}[$L^p$-universality of \ACFHINN{} for compactly supported $\CinftyOneDimTriangular$]
\label{lem:appendix:lp-univ for ACF}
Let \(p \in [1, \infty)\).
Assume \(\ACFINNUniversalClass\) is a \(\sup\)-universal approximator for \(\CcinftyRDminus\) and that it consists of piecewise \(C^1\)-functions.
Let $f \in \CzeroOneDimTriangular$, $\varepsilon>0$, and $K\subset \R^d$ be a compact subset.
Then, there exists $g \in \FACFHINN$ such that $\LpKnorm{f - g} < \varepsilon$.
\end{lemma}

\begin{proof}
Since we can take $a>0$, $b\in \R$ satisfying $aK+b \subset [0,1]^d$,
it is enough to prove the assertion for the case $K=[0,1]^d$.

Next, we show that we can assume that
for any $(\bm{x},y)\in \R^d$, $u(\bm{x},0)=0$ and $u(\bm{x},1)=1$ for any $\bm{x}\in \R^{d-1}$. 
Since $u(\bm{x}, \cdot)$ is a diffeomorphism, we have $u(\bm{x}, 0) \not = u(\bm{x}, 1)$ for any $x\in \mathbb{R}$.
By the continuity of $f$, either of $u(\bm{x}, 0) > u(\bm{x}, 1)$ for all $\bm{x}\in [0, 1]^{d-1}$ or $u(\bm{x}, 0) < u(\bm{x}, 1)$ for all $x\in [0, 1]^{d-1}$ holds.
Without loss of generality, we assume the latter case holds (if the former one holds, we just switch $u(\bm{x},0)$ and $u(\bm{x},1)$).
 We define $s(\bm{x}) = -\log(u(\bm{x}, 1) - u(\bm{x}, 0))$ and $t(\bm{x}) = -u(\bm{x}, 0)(u(\bm{x}, 1)-u(\bm{x}, 0))^{-1}$.
By a direct computation, we have
\[
    \Psi_{d-1, s, t} \circ f(\bm{x}, y) = \left(\bm{x}, \frac{u(\bm{x}, y) - u(\bm{x}, 0)}{u(\bm{x}, 1) - u(\bm{x}, 0)}\right) =: (\bm{x}, u_0(\bm{x}, y)).
\]
In particular, $\Psi_{s,t} \circ f(\bm{x}, 0) = (\bm{x}, 0)$ and $\Psi_{s,t}\circ s (\bm{x}, 1) = (\bm{x}, 1)$ hold.
, and the map $y \mapsto u_0(\bm{x}, y)$ is a diffeomorphism for each $\bm{x}$. 
Thus if we prove the existence of an approximator for $\Psi_{s,t} \circ f$, by Proposition \ref{prop: compatibility of approximation}, we can arbitrarily approximate $f$ itself.

For $\underline{k}:=(k_1,\dots,k_{d-1})\in\mathbb{Z}^{d-1}$ and $n\in \mathbb{N}$, we define $(\underline{k})_n := \sum_{i=1}^d k_i n^{i-1}\in \{0, \ldots, n^d-1\}$, that is, $\underline{k}$ is the $n$-adic expansion of $(\underline{k})_n$.
For any $n\in \mathbb{N}$, define the following discontinuous \ACF{}:
$\psi_n \colon [0,1]^d\to [0,1]^{d-1}\times [0,n^d]$ by 
\[\psi_n(\bm{x},y):=\left(\bm{x}, y+\sum_{k_1,\cdots, k_{d-1}=0}^{n-1} (\underline{k})_n 1_{\Delta^n_{\underline{k} + 1}}(\bm{x}) \right),\]
where $\underline{k} := (k_1, \ldots, k_d)$ and $\underline{k} + 1 := (k_1+1, \ldots, k_d+1)$.
We take an increasing function $v_n\colon \R\to \R$ that is 
smooth outside finite points such that 
\[
v_n(z):=
\begin{cases}
u\left(\frac{k_1}{n},\cdots, \frac{k_{d-1}}{n}, z-(\underline{k})_n\right)+(\underline{k})_n & \text{ if }z\in [(\underline{k})_n, (\underline{k})_n + 1) \\
z &\text{ if }z\notin[0,n^d).
\end{cases}
\]
We consider maps $h_n$ on $[0,1]^{d-1}\times [0,n^d]$ 
and $f_n: [0,1]^d\to [0,1]^d$ defined by 
\begin{align*} 
h_n(\bm{x},z)&:=(\bm{x},v_n(z)),\\
f_n&:=\psi_n^{-1}\circ h_n\circ \psi_n. 
\end{align*}
Then we have the following claim.\\
\textbf{Claim.} For all $k_1, \cdots, k_{d-1}=0,\cdots, n-1$, we have
\begin{equation*}
    f_n(\bm{x},y)=\left(\bm{x},u\left(\frac{k_1}{n},\ldots, \frac{k_{d-1}}{n}, y\right)\right)
\end{equation*}
on $\prod_{i=1}^{d-1}[\frac{k_i}{n},\frac{k_i+1}{n})\times [0,1)$.

In fact, we have
\begin{align*}
f_n(\bm{x},y)&=\psi_n^{-1}\circ h_n\circ \psi_n(\bm{x},y)\\
&=\psi_n^{-1}\circ h_n(\bm{x},y+(\underline{k})_n)\\
&=\psi_n^{-1}(\bm{x},v_n(y+(\underline{k})_n))\\
&=\psi_n^{-1}\left(\bm{x},u\left(\frac{k_1}{n},\ldots, \frac{k_{d-1}}{n},y\right)+(\underline{k})_n\right)\\
&=\left(\bm{x},u\left(\frac{k_1}{n},\ldots, \frac{k_{d-1}}{n}, y\right)\right). 
\end{align*}
Therefore, the claim above has been proved. 
Hence we see that $\supKnorm{f-f_n}\rightarrow 0$ as $n\rightarrow\infty$.
By Lemma~\ref{lem: univ. approx.} below and the universal approximation property of \(\ACFINNUniversalClass\), for any compact subset $K$ and $\varepsilon > 0$, there exist $g_1, g_2, g_3\in \FACFHINN$ such that $\LpKnorm{g_1-\psi_n^{-1}}<\varepsilon$, $\LpKnorm{g_2-h_n}<\varepsilon$, and $\LpKnorm{g_3-\psi_n}<\varepsilon$. Thus by Proposition \ref{prop: compatibility of approximation}, for any compact $K$ and $\varepsilon>0$, there exists $g\in \FACFHINN$ such that $\LpKnorm{g-f}<\varepsilon$.
\end{proof}

\subsection{Special case: Approximation of coordinate-wise independent transformation}
\label{sec:appendix:coordinate-wise independent}
In this section, we show the lemma claiming that special cases of single-coordinate transformations, namely coordinate-wise independent transformations, can be approximated by the elements of \ACFHINN{} given sufficient representational power of $\ACFINNUniversalClass$.

\begin{lemma}
\label{lem: univ. approx.}
Let \(p \in [1, \infty)\).
Assume \(\ACFINNUniversalClass\) is a \(\sup\)-universal approximator for \(\CcinftyRDminus\) and that it consists of piecewise \(C^1\)-functions.
Let $u:\mathbb{R}\rightarrow\mathbb{R}$ be a continuous increasing function.
Let $f:\mathbb{R}^d\rightarrow\mathbb{R}^d;(\bm{x}, y)\mapsto(\bm{x}, u(y))$ where $\bm{x}\in \R^{d-1}$ and $y\in\R$.
For any compact subset $K\subset\mathbb{R}^d$ and $\varepsilon>0$, there exists $g\in \FACFHINN$ such that $\LpKnorm{f - g} < \varepsilon$.
\end{lemma}
\begin{proof}
We may assume without loss of generality, in light of Lemma~\ref{smoothing lemma}, that $u$ is a $C^\infty$-diffeomorphism on $\R$ and that the inequality $u'(y)>0$ holds for any $y\in \mathbb{R}$.
Furthermore, we may assume that $u$ is compactly supported (i.e., $u(y) = y$ outside a compact subset of $\Re$) without loss of generality because we can take a compactly supported diffeomorphism $\tilde u$ and $a, b \in \Re$ ($a \neq 0$) such that $a \tilde{u} + b = u$ on any compact set containing $K$ by Lemma~\ref{red to comp. supp. diff}, and the scaling $a$ and the offset $b$ can be realized by the elements of \ACFHINN{}.

Fix $\delta \in (0,1)$.  We define the following functions:
\begin{align*}
\psi_0(\x,y):&=(\upto{d-2}{\x}, u'(y) x_{d-1}, y)\\
&= (\upto{d-2}{\x}, \exp(\log u'(y)) x_{d-1}, y),\\
\psi_1(\x,y):&=\left(\upto{d-2}{\x}, x_{d-1} + \delta^{-1}(u(y) - y), y\right),\\
\psi_2(\x,y):&=(\upto{d-2}{\x}, x_{d-1}, y+\delta x_{d-1}),\\
\psi_3(\x,y):&=\left(\upto{d-2}{\x}, x_{d-1}-\delta^{-1}(y-u^{-1}(y)), y\right),
\end{align*}
where we denote $\bm{x}=(x_1,\dots,x_{d-1})\in\R^{d-1}$.
First, we show that $\supKnorm{f - \psi_3\circ\psi_2\circ\psi_1\circ\psi_0} \to 0$ as $\delta \to 0$.
By a direct computation, we have
\begin{align*}
    \psi_3\circ\psi_2\circ\psi_1(\x,y)
    &= \psi_3\circ\psi_2(\upto{d-2}{\x}, x_{d-1} +\delta^{-1}(u(y) - y), y)\\
    &= \psi_3(\upto{d-2}{\x}, x_{d-1}+\delta^{-1}(u(y) - y), y+\delta(x_{d-1}+\delta^{-1}(u(y)-y)))\\
    &= \psi_3(\upto{d-2}{\x}, x_{d-1}+\delta^{-1}(u(y)-y),\delta x_{d-1} + u(y))\\
    &= (\upto{d-2}{\x}, x_{d-1}-\delta^{-1}(\delta x_{d-1}+u(y)-u^{-1}(\delta x_{d-1} + u(y))), \delta x_{d-1} + u(y))\\
    &= (\upto{d-2}{\x}, \delta^{-1}u^{-1}(\delta x_{d-2} + u(y)) - \delta^{-1}y, u(y) + \delta x_{d-1}),
\end{align*}
where $\x=(x_1,\dots,x_{d-1})\in\R^{d-1}$.
Since $u\in C^\infty([-r, r])$ where $r = \max_{(\x, y) \in K} |y|$, by applying Taylor's theorem, there exists a function $R(\x, y; \delta)$ and $C = C([-r, r], u) > 0$ such that
\[u^{-1}(u(y)+\delta x)=y+u'(y)^{-1}\delta x + R(\x,y; \delta)(\delta x)^2 \quad\text{ and }\quad \sup_{\delta\in (0,1)}|R(\x, y; \delta)| \le C\]
for all $(\x, y)\in K$.
Therefore, we have 
\[\psi_3\circ\psi_2\circ\psi_1\circ\psi_0(\x, y)=(\bm{x},u(y))+\delta(R(\x,u'(y)x_{d-1};\delta)\upto{d-1}{\x},u'(y)x_{d-1}).\]
For any compact subset $K$, the last term uniformly converges to 0 as $\delta\rightarrow 0$ on $K$.

Assume $\delta$ is taken to be small enough.
Now, we approximate $\psi_3 \circ \cdots \circ \psi_0$ by the elements of \ACFHINN{}.
Since $u$ is a compactly-supported $C^\infty$-diffeomorphism on $\Re$, the functions $(\upto{d-2}{\x}, y) \mapsto \log u'(y)$, $(\upto{d-2}{\x}, y) \mapsto u(y) - y$, and $(\upto{d-2}{\x}, y) \mapsto y - u^{-1}(y)$, each appearing in $\psi_0$, $\psi_1$, $\psi_3$, respectively, belong to \(\CcinftyRDminus\).
On the other hand, $\psi_2$ can be realized by $\FGL \subset \FLin$.
Therefore, combining the above with the fact that $\ACFINNUniversalClass$ is a \(\sup\)-universal approximator for \(\CcinftyRDminus\), we have that for any compact subset $K'\subset\ReD$ and any $\varepsilon > 0$, there exist $\phi_0$, $\ldots, \phi_3\in\FACFHINN$ such that $\supRangenorm{K'}{\psi_i - \phi_i} < \varepsilon$.
In particular, we can find $\phi_0, \ldots, \phi_3\in\FACFHINN$ such that $\LpRangenorm{K'}{\psi_i - \phi_i} < \varepsilon$.

Now, recall that $\ACFINNUniversalClass$ consists of piecewise $C^1$-functions as well as $\psi_i$ ($i =0, \ldots, 3$).
Moreover, $\psi_0, \psi_1, \psi_3$ are compactly supported while $\psi_2 \in \FGL$, hence they are Lipschitz continuous outside a bounded open subset.
Therefore, by Proposition~\ref{prop: compatibility of approximation}, we have the assertion of the lemma.

\end{proof}

\section{Locally bounded maps and piecewise diffeomorphisms}
\label{sec:appendix:piecewise diffeo}
In this section, we provide the notions of locally bounded maps and piecewise $C^1$-maps. These notions are used to state the regularity conditions on the CF layers in Theorem~\ref{theorem:main:1} and to prove the results in Section~\ref{appendix: compatibility of approximation and composition}.

\subsection{Definition of locally bounded maps}
Here, we provide the definition of locally bounded maps.
It is a very mild condition that is satisfied in most cases of practical interest, e.g., by continuous maps.

\begin{definition}[Locally bounded maps]
Let $f$ be a map from $\Re^m$ to $\Re^n$.
We say $f$ is \emph{locally bounded} if for each point $\x \in \Re^m$, there exists a neighborhood $U$ of $\x$ such that $f$ is bounded on $U$.
\end{definition}
As a special case, continuous maps are locally bounded; take an open ball $U$ centered at $\x$ and take a compact set containing $U$ to see that $f$ is bounded on $U$.

\subsection{Definition and properties of piecewise $C^1$-mappings}
In this section, we give the definition of piecewise $C^1$-mappings and their properties.
Examples of piecewise $C^1$-diffeomorphisms appearing in the paper include the \(\FSACFH\) with \(\ACFINNUniversalClass\) being MLPs with ReLU activation.  
\begin{definition}[piecewise $C^1$-mappings]\label{def:piecewiseC1}
Let $f:\Re^m\rightarrow\Re^n$ be a measurable map.  We say $f$ is a \emph{piecewise $C^1$-mapping} if there exists a mutually disjoint family of (at most countable)
open subsets $\{V_i\}_{i\in I}$ such that
\begin{itemize}
    \item ${\rm vol}(\ReD\setminus U_f)=0$,
    \item for any $i\in I$, there exists an open subset $W_i$  containing the closure $\overline{V_i}$ of $V_i$, and  $C^1$-mapping $\tilde{f}_i:W_i\rightarrow\ReD$ such that $\tilde{f}_i|_{V_i}=f|_{V_i}$, and
    \item for any compact subset $K$, $\#\{i\in I : V_i\cap K\neq \emptyset\}<\infty$.
\end{itemize}
where we denote $U_f:=\bigsqcup_{i\in I} V_i$, and \(\#(\cdot)\) denotes the cardinality of a set. 
\end{definition}
We remark that piecewise $C^1$-mappings are essentially locally bounded in the sense that for any compact set $K\subset\ReD$, ${\rm ess.sup}_{K}\Vert f\Vert=\supRangenorm{K\cap U_f}{f} <\infty$.
Then we define a piecewise $C^1$-diffeomorphisms:

\begin{definition}[piecewise $C^1$-diffeomorphisms]
Let $f:\ReD\rightarrow\ReD$ be a piecewise $C^1$-mapping.  We say $f$ is a \emph{piecewise $C^1$-diffeomorphism} if 
\begin{enumerate}
    \item the image of nullset via $f$ is also a nullset,
    \item $f|_{U_f}$ is injective, and for $i\in I$, $\tilde{f}_i$ is a $C^1$-diffeomorphism from $W_i$ onto $\tilde{f}_i(W_i)$,
    \item ${\rm vol}\left(\ReD\setminus f(U_f)\right)=0$, and
    \item for any compact subset $K$, $\#\{i\in I : f(V_i)\cap K\neq \emptyset\}<\infty$.
\end{enumerate}
\end{definition}
We summarize the basic properties of piecewise $C^1$-diffeomorphisms in the proposition below:
\begin{proposition}
\label{prop: basic properties}
Let $f$ and $g$ be piecewise $C^1$-diffeomorphisms. Then, we have the following.
\begin{enumerate}
    \item There exists a piecewise $C^1$-diffeomorphism $f^\dagger$ such that $f( f^\dagger(x))=x$ for $x\in U_{f^\dagger}$ and $f^\dagger(f(y))=y$ for $y\in U_f$\label{existence of inverse}.
    \item For any $h\in L^1$, we have $\int h(x) dx =\int h(f(x))|Df(x)|dx$, where $|Df(x)|$ is the absolute value of the determinat of the Jacobian matrix of $f$ at $x$.  \label{change of var}
    \item For any compact subset $K$, $f^{-1}(K)\cap U_f$ is a bounded subset.\label{compfinvbdd}
    \item For any nullset $F$, then $f^{-1}(F)$ is also a nullset.  \label{finvzero}
    \item For any measurable set $E$ and any compact set $K$, $f^{-1}(E\cap K)$ has a finite volume.\label{finvfin}
    \item The composition $f\circ g$ is also a piecewise $C^1$-diffeomorphism. \label{composition}
\end{enumerate}
\end{proposition}

\begin{proof}
\textit{Proof of \ref{existence of inverse} }:
Fix $a\in\ReD$.
For $x\in \ReD\setminus f(U_f)$, define $f^\dagger(x)=a$, and for $x\in f(V_i)$, define $f^\dagger(x):=f|_{V_i}^{-1}(x)$. Then, $f^\dagger$ is a piecewise $C^1$-mapping with respect to the family of pairwise disjoint open subsets $\{f(V_i)\}_{i\in I}$, and satisfies the conditions for piecewise $C^1$-diffeomphism.

\textit{Proof of \ref{change of var} }: 
It follows by the following computation:
\begin{align*}
    \int h(x) dx&=\int_{f(U_f)}h(x)dx\\
    &=\sum_{i\in I}\int_{f(V_i)}h(x) dx\\
    &=\sum_{i\in I}\int_{V_i}h(f(x))|Df(x)|dx=\int h(f(x))|Df(x)|dx.
\end{align*}

\textit{Proof of \ref{compfinvbdd}}
It suffices to show that $f^{-1}(K)\cap U_f$ is covered by finitely many compact subsets. In fact, we remark that only finitely many $V_i$'s intersect with $f^{-1}(K)$. If not, infinitely many $f(V_i)$ intersects $f(f^{-1}(K))\subset K$, which contradicts the definition of piecewise $C^1$-diffeomorphisms. Let $I_0\subset I$ be a finite subset composed of $i\in I$ such that $V_i$ intersecting with $f^{-1}(K)$. For $i\in I_0$, we define a compact subset $F_i:=\tilde{f}_i^{-1}(\tilde{f}_i(\overline{V_i})\cap K)$.  Then we see that $f^{-1}(K)\cap U_f$ is contained in $\cup_{i\in I_0}F_i$. 

\textit{Proof of \ref{finvzero} }: 
It suffices to show that for any compact subset $K$, the volume of $f^{-1}(F)\cap K$ is zero.
By applying \ref{change of var} to the case $h=\Indicator{F}$, we see that
\[\int_{f^{-1}(F)}|Df(x)|dx=0.\]
For  $n>0$, let $E_n:=f^{-1}(F)\cap K \cap \{x\in\ReD:|Df(x)|\ge 1/n\}$. Then we have
\[\frac{{\rm vol}(E_n)}{n}\le \int_{E_n}|Df(x)|dx \le  \int_{f^{-1}(F)}|Df(x)|dx=0,\]
thus ${\rm vol}(K\cap f^{-1}(F))=\lim_{n\rightarrow\infty}{\rm vol}(E_n)=0$

\textit{Proof of \ref{finvfin} }:
By applying \ref{change of var} to the case $h=\Indicator{E\cap K}$, we see that
\[\int_{f^{-1}(E\cap K)}|Df(x)|dx= {\rm vol}(E\cap K).\]

Let $F$ be a closure of $f^{-1}(K)\cap U_f$.  By \ref{compfinvbdd}, $F$ is a compact subset.  Let $I_0:=\{i\in I : F\cap V_i\neq\emptyset\}$ be a finite subset. Then we have 
\begin{align*}
    C&:=\inf_{f^{-1}(K)\cap U_f}|Df|\\
    &\ge \inf_{i\in I_0}\inf_{F\cap\overline{V_i}}|D\tilde{f}_i| >0.
\end{align*}
Thus,
\begin{align*}
\int_{f^{-1}(E\cap K)\cap U_f}|Df(x)|dx 
\geq C{\rm vol}(f^{-1}(E\cap K)),
\end{align*}
where the last equality follows from ${\rm vol}(f^{-1}(E\cap K)\setminus U_f)=0$.
Thus we have ${\rm vol}(f^{-1}(E \cap K))<\infty$

\textit{Proof of \ref{composition} }: 
We denote by $\{V_i\}_{i\in I}$, $\{V'_j\}_{j\in J}$ the disjoint open families associated with $f$ and $g$, respectively. At first, we prove $f\circ g$ is a piecewise $C^1$-mapping.  Let $V_{ij}:=g^{-1}(V_i\cap g(V'_j))\cap U_g$ and define $\mathcal{U}_{f\circ g}:=\{V_{ij}\}_{(i,j)\in I\times J}$. Let $U_{f\circ g}:=\cup_{i,j}V_{ij}=g^{-1}(U_f\cap g(U_g))\cap U_g$. By \ref{finvzero}, the volume of $\ReD\setminus U_{f\circ g}$ is zero.  On each $V_{ij}$, $\tilde{f}_i\circ\tilde{g}_j$ is an extension of $f\circ g|_{V_{ij}}$.  For any compact subset $K$, $\#\{(i,j)\in I\times J : K\cap V_{ij}\neq \emptyset\}<\infty$. In fact, suppose the number is infinite.  Then $g(U_f\cap K)$ intersects with an infinite number of open subsets in the form of $g(U_f\cap K)\cap V_i\cap g(V'_j)$.  On the other hand $g(U_f\cap K)$ is a bounded subset, thus by definition, the number of $(i,j)\in I\times J$ satisfying $\overline{g(U_f\cap K)}\cap V_i\cap g(V'_j)\neq \emptyset$ is finite. It is a contradiction. Therefore, $g\circ f$ is a piecewise $C^1$-mapping.  

Next, we prove $f\circ g$ is a piecewise $C^1$-diffeomorphism.  
The first and second condition follows by definition.  For the third condition, since $\ReD\setminus f\circ g(U_{f\circ g})=\big(\ReD\setminus f(U_f)\big)\cup \big(\ReD\setminus f\big(g(U_g)\big)\subset \ReD\setminus f(g(U_g)\cap U_f)$, it suffices to show that the volue of $\ReD\setminus f(g(U_g)\cap U_f)$ is zero.  In fact, by the injectivity of $f$ on $U_f$, we have $f(g(U_g)\cap U_f)=f(U_f)\setminus f(U_f\setminus g(U_g))$. Thus $\ReD\setminus f(g(U_g)\cap U_f)=(\ReD\setminus f(U_f))\cup f(U_f\setminus g(U_g))$. By definition of $C^1$-diffeomorphism, we conclude $\ReD\setminus f(g(U_g)\cap U_f)$ is a nullset.  
For the fourth condition, let $K$ be a compact subset.  
Assume the $\{(i,j)\in I\times J : f\circ g(V_{ij})\cap K\neq\emptyset\}=\infty$.  
Since $f$ is a piecewise $C^1$-diffeomorphism, there exist infinitely many elements in $j\in J$ such that $f\circ g(V'_j)\cap f(U_f)\cap K\neq \emptyset$. 
On the other hand, $f^{-1}(K\cap f(U_f))\cap U_f$ is bounded, and its closure intersects with only finitely many $g(V'_j)$'s, thus $ K\cap f(U_f)$ intersects with only finitely many $f\circ g(V'_j)$, which is a contradiction.

\end{proof}
For a measurable mapping $f: \R^m \to \R^n$ and any $R>0$, we define a measurable set 
\[\mathcal{L}(R;f):=\{x\in \R^m : \Vert f(x)-f(y) \Vert > R\Vert x-y\Vert \text{ for some $y\in U_f$}\}.\]
Then we have the following proposition:
\begin{proposition}
\label{prop: weak lipschitz}
Let $f:\Re^m\rightarrow\Re^n$ be a piecewise $C^1$-mapping. Assume $f$ is linearly increasing, namely, there exists $a,b>0$ such that $\Vert f(x)\Vert <a\Vert x\Vert + b$ for any $x\in \R^m$.
Then for any compact subset $K'$, ${\rm vol}(\mathcal{L}(R;f)\cap K')\rightarrow 0$ as $R\rightarrow\infty$.
\end{proposition}
\begin{proof}
Let $B$ be an open ball containing $K'$ of radius $r$. Fix an arbitrary $\varepsilon>0$.  We note that the linearly increasing condition implies the locally boundedness of $f$.
Let $C:=\sup_{\overline{B}}\Vert f\Vert$.
For $\delta>0$, we define 
\[V_\delta:=\{x\in \overline{B}: {\rm dist}\left(x,\partial{U_f}\cup\partial{B})\right)<\delta\},\] where ${\rm dist}(x,S):=\inf_{y\in S}\{\Vert x-y\Vert\}$.  
Set $\delta$ to be ${\rm vol}(V_\delta)<\varepsilon$. 
We claim that 
\[L:=\sup_{(x,y)\in K'\times \Re^m\setminus B}\frac{\|f(x)-f(y)\|}{\|x-y\|}\]
is finite. In fact, let $r':=\inf_{x\in K', y\notin B}\|x-y\|$. Then for $x\in K'$ and $y\notin B$, we have
\begin{align*}
    \frac{\|f(x)-f(y)\|}{\|x-y\|}
    &\le \frac{\|f(x)\|+ \|f(y)\|}{\|x-y\|} \\
    &\le \frac{a \|x\| + a\|y\| + 2b }{\|x-y\|} \\    
    &\le \frac{a \|x\| + a(\|x - y\| + \|x\|)+ 2b }{\|x-y\|} \\
    &\le a+\frac{2a\|x\|+2b}{\|x-y\|}\\
    &< a+\frac{2ar+2b}{r'}.
\end{align*}
Thus, $L$ is finite.
Since $\overline{B}$ intersects with finitely many $V_i$'s, $f|_{B\setminus V_{\delta/2}}$ is a Lipschitz function.  Put $L_\delta>0$ as the Lipschitz constant of $f|_{B\setminus V_{\delta/2}}$. 
Then for any $R>\max(L, L_\delta, 4C/\delta)$, we see that $\mathcal{L}(R;f)\cap K'$ is contained in $V_\delta$.
Actually, we should prove that $x\not \in \mathcal{L} (R;f)$ when $x \in K'\setminus V_{\delta}$.
Take arbitrary $y\in \mathbb{R}^m$. When $y\not \in B$, since $x\in K'$, we have $\frac{\|f(x)-f(y)\|}{\|x-y\|}\leq L$ by the definition of $L$.
When $y \in B\setminus V_{\delta/2}$, since $x\in K'\setminus V_{\delta} \subset B\setminus V_{\delta/2}$, we have $\frac{\|f(x)-f(y)\|}{\|x-y\|}\leq L_{\delta}$ by the definition of $L_{\delta}$.
When $y\in V_{\delta/2}$, we have $\|x-y\|\geq \frac{\delta}{2}$ because $x\not \in V_{\delta}$. Thus,
\begin{align*}
    \frac{\|f(x)-f(y)\|}{\|x-y\|} \leq \frac{\|f(x)\| + \|f(y)\|}{\delta/2} \leq \frac{C + C}{\delta/2} \leq \frac{4C}{\delta}.
\end{align*}
Combining these three cases, we conclude that $x\not \in \mathcal{L}(R;f)$.
Thus we have ${\rm vol}(\mathcal{L}(R;f)\cap K')<\varepsilon$, namely, we conclude ${\rm vol}(\mathcal{L}(R;f)\cap K')\rightarrow 0$ as $R\rightarrow\infty$.
\end{proof}
\begin{remark}
   The linearly increasing condition is important to prove our main theorem. Our approximation targets are compactly supported diffeomorphisms, affine transformations, and the discontinuous \ACFs{} appeared in Section \ref{subsec: approximation via ACF} or Section \ref{sec:appendix:lp ACFINN approx general}, all of which satisfy the linearly increasing condition.
\end{remark}

\section{Compatibility of approximation and composition}
\label{appendix: compatibility of approximation and composition}
In this section, we prove the following proposition.
It enables the component-wise approximation, i.e., given a transformation that is represented by a composition of some transformations, we can approximate it by approximating each constituent and composing them.
The justification of this procedure is not trivial and requires a fine mathematical argument.
The results here build on the terminologies and the propositions for piecewise $C^1$-diffeomorphisms presented in Section~\ref{sec:appendix:piecewise diffeo}.

\begin{proposition}
\label{prop: compatibility of approximation}
Let \(\ARFINNModel\) be a set of piecewise $C^1$-diffeomorphisms (resp. locally bounded maps) from $\ReD$ to $\ReD$,
and $F_1,\dots,F_r$ be linearly increasing piecewise $C^1$-diffeomorphisms
(resp. continuous maps) from $\ReD$ to $\ReD$ ($r\ge 2$).  
Assume for any $\varepsilon>0$ and compact set $K\subset\mathbb{R}^d$, there exists $\widetilde{G}_1,\dots, \widetilde{G}_r\in \ARFINNModel$ such that for $i\in[r]$, $\big\Vert F_i-\widetilde{G}_i\big\Vert_{p, K}<\varepsilon$ (resp. $\big\Vert F_i-\widetilde{G}_i\big\Vert_{\sup, K}<\varepsilon$).  
Then for any $\varepsilon>0$ and compact set $K\subset\mathbb{R}^d$, there exists $G_1, \dots, G_r\in \ARFINNModel$, such that
\[\LpKnorm{F_r\circ\cdots \circ F_1-G_r\circ\cdots\circ G_1} < \varepsilon \]
\[\left(\text{resp. }\supKnorm{F_r\circ\cdots \circ F_1-G_r\circ\cdots\circ G_1} < \varepsilon\right.)\]
\end{proposition}
\begin{proof}
We prove by induction.
In the case of $r=2$, it follows by Lemma~\ref{lem:compose-approximations-Lp} (for \(L^p\)-norm) or Lemma~\ref{lemma:composition} (for \(\sup\)-norm) below in the case of $\ARFINNModel_1=\ARFINNModel_2=\ARFINNModel$.  In the general case, let $\widetilde{F}_2:=F_r\circ\cdots F_2$. Then by the induction hypothesis, for any compact set $K$ and $\varepsilon>0$, there exists $\widetilde{G}_2=G_r\circ\cdots\circ G_2$ for some $G_i\in\ARFINNModel$ such that $\big\Vert\widetilde{F}_2-\widetilde{G}_2\big\Vert_{?,K}<\varepsilon$, where $?=\text{$p$ or ${\sup}$}$. By applying Lemma~\ref{lem:compose-approximations-Lp} or Lemma~\ref{lemma:composition} with $\ARFINNModel_1=\ARFINNModel$ and $\ARFINNModel_2=\ARFINNModel\circ\cdots\circ \ARFINNModel$ (the set of compositions of $r-1$ elements of $\ARFINNModel$) below, we conclude the proof.
\end{proof}
\begin{lemma}
\label{lem:compose-approximations-Lp}
Let \(\ARFINNModel_1\) and $\ARFINNModel_2$ be sets of piecewise $C^1$-diffeomorphisms from $\ReD$ to $\ReD$. Let $F_1, F_2: \R^d\rightarrow\R^d$ be
linearly increasing piecewise $C^1$-diffeomorphisms.
Assume for any $\varepsilon>0$ and compact set $K\subset\mathbb{R}^d$, for $i=1,2$, there exists $\widetilde{G}_i\in \ARFINNModel_i$ such that $\LpKnorm{F_i-\widetilde{G}_i}<\varepsilon.$
Then for any $\varepsilon>0$ and compact set $K\subset\mathbb{R}^d$, for $i=1,2$, there exists $G_i\in\ARFINNModel_i$, such that
\[\LpKnorm{F_2\circ F_1-G_2\circ G_1} < \varepsilon.\]
\end{lemma}
\begin{proof}
Fix arbitrary $\varepsilon>0$ and compact set $K\subset\mathbb{R}^d$. 
Put $K':=\overline{F_1(K\cap U_{F_1})}$. 
Then, since $F_1(K\cap U_{F_1})$ is bounded (see the remark under Definition~\ref{def:piecewiseC1}), $K'$ is compact.
We claim that there exists $R>0$ such that
\[\vol{F_1^{-1}\left(\mathcal{L}(R;F_2)\cap K'\right)}^{1/p}<\frac{\varepsilon}{3\underset{K'}{\rm ess.sup}\Vert F_2\Vert},\]
which can be confirmed as follows. Take an increasing sequence $R_n>0$ $(n\geq 1)$ satisfying $\lim_{n\to \infty}R_n=\infty$.
Let $B_n := \mathcal{L}(R_n;f)\cap K'$ and $A_n := F_1^{-1}(B_n)$.
Then, from Proposition~\ref{prop: weak lipschitz}, we have $\vol{B_n}\to 0$, which implies  
$\vol{\bigcap_{n=1}^\infty B_n}=0$.
By Proposition~\ref{prop: basic properties} (\ref{finvzero}),
we have $\vol{\bigcap_{n=1}^\infty A_n} = \vol{F_1^{-1}(\bigcap_{n=1}^\infty B_n)}=0$.
By Proposition~\ref{prop: basic properties} (\ref{finvfin}), we have $\vol{A_1}=\vol{F_1^{-1}(B_1)}<\infty$.
Recall that if a decreasing sequence $\{S_n\}_{n=1}^\infty$ of measurable sets satisfies $\vol{S_1}<\infty$ and $\vol{\bigcap_{n=1}^\infty S_n}=0$, 
then $\lim_{n \to \infty}\vol{S_n} = 0$.
Therefore, we obtain $\lim_{n\to \infty}\vol{A_n}=0$ and we have the assertion of the claim.

Take  $G_1\in\ARFINNModel_1$ such that 
\[\LpKnorm{F_1-G_1} <\frac{\varepsilon}{3R}.\]

Put $S:=F_1^{-1}\left(\mathcal{L}(R;F_2)\cap K'\right)$, and define a compact subset $K'':=\overline{(G_1^\dagger)^{-1}(K)\cap U_{G_1^\dagger}}$. Here, the compactness of $K''$ follows from Proposition~\ref{prop: basic properties} (\ref{compfinvbdd}).   
Next, we take $G_2\in\ARFINNModel_2$ such that
\[\Vert F_2-G_2\Vert_{p,K''} <\frac{\varepsilon}{3\underset{{ (G_1^\dagger)^{-1}(K)}}{\rm ess.sup}|\det(DG_1^\dagger)|}\]
where $G_1^\dagger$ is a piecewise $C^1$-diffeomorphism defined by Proposition \ref{prop: basic properties} (\ref{existence of inverse}).
Then we have
\begin{align*}
&\LpKnorm{F_2\circ F_1-G_2\circ G_1}\\
&\le\LpKnorm{F_2\circ F_1-F_2\circ G_1} + \LpKnorm{F_2\circ G_1-G_2\circ G_1}\\
&\le\LpKnorm{(F_2\circ F_1-F_2\circ G_1)\mathbf{1}_S}+ \LpKnorm{(F_2\circ F_1-F_2\circ G_1)\mathbf{1}_{K\setminus S} }\\
& \hspace{11pt}+ \underset{{(G_1^\dagger)^{-1}(K)}}{\rm ess.sup}|\det(DG_1^{\dagger})|\Vert F_2-G_2\Vert_{p,K''}\\
&<\varepsilon.
\end{align*}
\end{proof}

\begin{lemma}[compatibility of composition]\label{lemma:composition}
Let \(\ARFINNModel_1\) and $\ARFINNModel_2$ be sets of locally bounded maps from $\ReD$ to $\ReD$. Let $F_1, F_2: \R^d\rightarrow\R^d$ be continuous maps.  Assume for any $\varepsilon>0$ and compact set $K\subset\mathbb{R}^d$, for $i=1,2$, there exists $\widetilde{G}_i\in \ARFINNModel_i$ such that  $\supKnorm{F_i-\widetilde{G}_i}<\varepsilon$.
Then for any $\varepsilon>0$ and compact set $K\subset\mathbb{R}^d$, for $i=1,2$, there exists $G_i\in \ARFINNModel_i$, such that
\[\supKnorm{F_2\circ F_1-G_2\circ G_1} < \varepsilon.\]
\end{lemma}
\begin{proof}
Take any positive number $\epsilon>0$ and compact set $K\subset\R^d$. 
Put $r:=\max_{k\in K}|F_1(k)|$ and 
$K':=\{x\in \R^d: |x|\leq r+1\}$. 
Let $G_2\in \ARFINNModel_2$  satisfying 
\begin{align*}
\sup_{x\in K'}|F_2(x)-G_2(x)| \leq  \frac{\epsilon}{2}.  \end{align*}
Since any continuous map is uniformly continuous on a compact set, we can take a positive number $\delta>0$ such that for any $x, y\in K'$ with $|x-y|<\delta$, 
\[ |F_2(x)-F_2(y)|<\frac{\varepsilon}{2}. \]
From the assumption, we can take $G_1\in \ARFINNModel_1$ satisfying
\begin{align*} \sup_{x\in K}|F_1(x)-G_1(x)|\leq \min\{1,\delta\}. \end{align*}
Then, it is clear that $F_1(K)\subset K'$ by the definition of $K'$.
Moreover, we have $G_1(K)\subset K'$. 
In fact, we have 
\begin{align*}
|G_1(k)| \leq \sup_{x\in K}|F_1(x)-G_1(x)|+|F_2(k)|\leq 1+r \quad (k\in K). 
\end{align*}
Then for any $x\in K$, we have
\begin{align*}
|F_2\circ F_1(x)-G_2\circ G_1(x)|
&\leq |F_2(F_1(x))-F_2(G_1(x))| + | F_2(G_1(x)) -G_2(G_1(x))|\\
&<\epsilon. 
\end{align*}
\end{proof}
 
\section{Examples of flow architectures covered in this paper}
\label{appendix:sec:examples}
Here, we provide the proofs for the universal approximation properties of certain \ARFINNs{}.

\subsection{Neural autoregressive flows (NAFs)}
In this section, we prove that \emph{neural autoregressive flows} \cite{HuangNeural2018b} yield $\sup$-universal approximators for $\ConeOneDimTriangular$ (hence for $\CinftyOneDimTriangular$).
The proof is not merely an application of a known result in \citet{HuangNeural2018c} but it requires additional non-trivial consideration to enable the adoption of Lemma~3 in \citet{HuangNeural2018c} as it is applicable only for those smooth mappings that match certain boundary conditions.
\begin{definition}
A \emph{deep sigmoidal flow} (DSF; a special case of neural autoregressive flows) \citep[Equation~(8)]{HuangNeural2018b} is a flow layer \(g = (g_1, \ldots, g_d) \colon \ReD \to \ReD\) of the following form:
\begin{equation*}\begin{aligned}
g_k(\x) &:= \sigma^{-1} \left(\sum_{j=1}^{n} w_{k,j}(\upto{k-1}{\x}) \cdot \sigma\left(\frac{x_k - b_{k,j}(\upto{k-1}{\x})}{\tau_j(\upto{k-1}{\x})}\right)\right),
\end{aligned}\end{equation*}
where \(\sigma\) is the sigmoid function, \(n \in \Na\), \(w_j, b_j, \tau_j \colon \Re^{k-1} \to \Re\) (\(j \in [n]\)) are neural networks such that \(b_j(\cdot) \in (r_0, r_1)\), \(\tau_j(\cdot) \in (0, r_2)\), \(w_j(\cdot) > 0\), and \(\sum_{j=1}^{n} w_j(\cdot) = 1\) (\(r_0, r_1 \in \Re\), \(r_2 > 0\)).
We define \(\FDSF\) to be the set of all possible DSFs.
\end{definition}
\begin{proposition}\label{prop:DSF}
The elements of \(\FDSF\) are locally bounded, and \(\INN{\FDSF}\) is a \(\sup\)-universal approximator for \(\ConeOneDimTriangular\).
\end{proposition}
\begin{proof}
The elements of \(\FDSF\) are continuous, hence locally bounded.
Let $s=(s_1,\cdots, s_d)\in \ConeOneDimTriangular$. 
Take any compact set $K\subset \R^d$ and $\epsilon>0$. 
Since $K$ is compact, there exist $r_0, r_1\in \R$ such that $K\subset [r_0,r_1]^d$. 
Put $r_0'=r_0-1$, $r_1'=r_1+1$. 
We take a $C^1$-function $b\colon (r_0', r_1')\to \R$ satisfying 
\begin{enumerate}
    \item $b|_{[r_0,r_1]}=0$, 
    \item $b|_{(r_0',r_0)}$ and $b|_{(r_1,r_1')}$ are strictly increasing, 
    \item $\lim_{x\to r_0'+0}b(x)=-\infty$ and $\lim_{x\to  r_1'-0}b(x)=\infty$,
    \item $\lim_{x\to r_0'+0}\frac{d(\sigma\circ b)}{dx}(x)$ and  $\lim_{x\to  r_1'-0}\frac{d(\sigma\circ b)}{dx}(x)$ exist in $\R$, 
\end{enumerate}
where $\sigma$ is the sigmoid function.
For each $k \in [d]$, we define a $C^1$-map $\tilde{s}_k\colon [r_0',r_1']^{k-1}\times (r_0',r_1')\times [r_0', r_1']^{d-k}\to \R$, which is strictly increasing with respect to $x_k$, by 
\[
\tilde{s}_k(x):=s_k(x)+b(x_k)\quad (x=(x_1,\cdots, x_d)). 
\]
Moreover, we define a map $S\colon [r_0',r_1']^d\to [0,1]^d$ by 
\begin{align*}
    S_k|_{[r_0',r_1']^{k-1}\times (r_0',r_1')\times [r_0', r_1']^{d-k}}&=\sigma \circ \tilde{s}_k,\\
    S_k(x_1,\cdots, x_{k-1}, r_0', x_{k+1}, \cdots, x_d)&=0,\\
    S_k(x_1,\cdots, x_{k-1}, r_1', x_{k+1}, \cdots, x_d)&=1,
\end{align*}
where we write $S=(S_1,\cdots, S_d)$. 
Then, by Lemma~\ref{lem:appendix:DSF extended map is smooth}, $S$ satisfies the assumptions of Lemma~3 in \cite{HuangNeural2018c}. 
Since $S([r_0,r_1]^d)\subset (0,1)^d$ is compact, 
there exists a positive number $\delta>0$ such that 
\[ 
S([r_0,r_1]^d) + B(\delta):= \{S(x)+v \ :\ x\in [r_0,r_1]^d, v\in B(\delta)\} \subset [\delta,1-\delta]^d,
\]
where $B(\delta):=\{x\in \R^d : |x|\leq \delta\}$. 
Let $L>0$ be a Lipschitz constant of $\sigma^{-1}\colon (0,1)^d\to \R^d$ on $[\delta, 1-\delta]^d$. 
By Lemma~3 in \cite{HuangNeural2018c}, 
there exists $g\in \INN{\FDSF}$ such that 
\begin{align*}
    \|S-\sigma \circ g\|_{\sup, [r_0', r_1']^d}<\min\left\{\delta, \frac{\epsilon}{L}\right\}. 
\end{align*}
As a result, $\sigma\circ g([r_0,r_1]^d) \subset S([r_0,r_1]^d) + B(\delta) \subset [\delta, 1-\delta]^d$.
Then we obtain 
\begin{align*}
\|s-g\|_{\sup, K}
\leq 
\|s-g\|_{\sup, [r_0,r_1]^d}
&=\| \sigma^{-1} \circ \sigma\circ s- \sigma^{-1}\circ \sigma \circ g\|_{\sup, [r_0,r_1]^d}\\
&\leq L\| S -\sigma \circ g\|_{\sup, [r_0,r_1]^d}\\
&<\epsilon. 
\end{align*}
\end{proof}

\begin{lemma}
\label{lem:appendix:DSF extended map is smooth}
We denote by $\mathcal{T}^1$ the set of all $C^1$-increasing triangular mappings from $\R^d$ to $\R^d$. 
For $s=(s_1,\cdots, s_d)\in \mathcal{T}^1$, we define 
a map $S\colon [r_0',r_1']^d\to [0,1]^d$ as in the proof of Proposition~\ref{prop:DSF}. 
Then $S$ is a $C^1$-map. 
\end{lemma}
\begin{proof}
It is enough to show that 
$S_d\colon [r_0', r_1']^d\to [0,1]$ is a $C^1$-function. 
We prove that for any $i\in [d]$, the $i$-th partial derivative of $S_d$ exists and that it is continuous on $[r_0', r_1']^d$. 
First, for $i\in [d-1]$, we consider the $i$-th partial derivative. \\
\textbf{Claim 1}. 
\begin{align*}
    \frac{\partial S_d}{\partial x_i}(x)=
    \begin{cases}
    \frac{d\sigma}{d x}(s_i(x)+b(x_d))\frac{\partial s_d}{\partial x_i}(x) & (x\in [r_0',r_1']^{d-1}\times (r_0', r_1'))\\
    0 & (x_d= r_0', r_1')
    \end{cases}
\end{align*}
In fact, 
for $x\in [r_0', r_1']^{d-1}\times (r_0',r_1')$, we have  
\begin{align*}
    \frac{\partial S_d}{\partial x_i}(x)
    =\frac{\partial (\sigma \circ \tilde{s_d})}{\partial x_i}(x)
    =\frac{d\sigma }{dx} (s_d(x)+b(x_d)) \left(\frac{\partial s_d}{\partial x_i}(x)+0\right). 
\end{align*}
For $x=(x_{\leq {d-1}}, r_0')$, we have 
\begin{align*}
    \frac{\partial S_d}{\partial x_i}(x)
    &=\lim_{h\to 0}\frac{S_d(x_{\leq i-1}, x_i+h, x_{i+1},\cdots, x_{d-1}, r_0')-S_d(x_{\leq d-1}, r_0')}{h}\\
    &=\lim_{h\to 0}\frac{0-0}{h}=0
\end{align*}
Here, note that by the definition of $S_d$, the notation $S_d(x_{\leq i-1}, x_i+h, x_{i+1},\cdots, x_{d-1}, r_0')$ makes sense even if $x_i=r_0'$ or $x_i=r_1'$. 
We can verify the case $x=(x_{\leq d-1}, r_1')$ similarly.

Next, we show that $\frac{\partial S_d}{\partial x_i}$ is continuous. 
We take any $x_{\leq d-1}\in [r_0',r_1']^{d-1}$.
Since we have $\lim_{x\to r_0'}b(x)=-\infty$, $\lim_{x\to r_1'}b(x)$,  $\lim_{x\to \pm \infty} \frac{d\sigma }{dx}(x)=0$, and 
$|\frac{\partial s_d}{\partial x_I}(x)|<\infty$ $(x\in [r_0', r_1']^d)$, 
we obtain 
\begin{align*}
    \lim_{x\to (x_{d-1}, r_0')}\frac{d\sigma}{d x}(s_i(x)+b(x_d))\frac{\partial s_d}{\partial x_i}(x)=0,\\
    \lim_{x\to (x_{d-1}, r_1')}\frac{d\sigma}{d x}(s_i(x)+b(x_d))\frac{\partial s_d}{\partial x_i}(x)=0.
\end{align*}
Therefore, the partial derivative $\frac{\partial S_d}{\partial x_i}(x)$ is continuous on $[r_0', r_1']^d$ for $i\in [d-1]$. 

Next, we consider the $d$-th derivative of $S_d$. \\
\textbf{Claim 2.}
\begin{align*}
    \frac{\partial S_d}{\partial x_d}(x)=
    \begin{cases}
    \frac{d\sigma}{d x}(s_d(x)+b(x_d)) \left( \frac{\partial s_d}{\partial x_d }(x)+\frac{d b}{d x}(x_d)\right) & (x\in [r_0', r_1']^{d-1}\times (r_0', r_1'))\\
    e^{s_d(x_{\leq d-1},r_0')}\lim_{x\to r_0'+0}\frac{d (\sigma\circ b)}{dx}(x) & (x_d=r_0')\\
    e^{-s_d(x_{\leq d-1}, r_1')}\lim_{x\to r_1'-0}\frac{d (\sigma\circ b)}{dx}(x) & (x_d=r_1')
    \end{cases}
\end{align*}
We verify Claim 2. 
Since it is clear for the case $x\in [r_0', r_1']^{d-1}\times (r_0',r_1')$ by the definition of $S_k$, we consider the case $x_d=r_0', r_1'$. \\
\textbf{Subclaim.}
For $x_{\leq d-1}'\in [r_0',r_1']^{d-1}$, 
\begin{align*}
    \lim_{x\to (x_{\leq d-1}', r_0')}&\frac{\sigma (s_d(x)+b(x_d))}{\sigma(b(x_d))}=e^{s_d(x_{\leq d-1}',r_0')}\\
    \lim_{x\to (x_{\leq d-1}', r_1')}&\frac{\sigma(s_d(x)+b(x_d))-1}{\sigma(b(x_d))-1}=e^{-s_d(x_{\leq d-1}', r_1')}
\end{align*}
We verify this subclaim. From $\lim_{x\to r_0'}b(x)=-\infty$, we have
\begin{align*}
    \frac{\sigma (s_d(x)+b(x_d))}{\sigma(b(x_d))}
    &=\frac{1+e^{-b(x_d)}}{1+e^{-s_d(x)-b(x_d)}}
    =\frac{e^{b(x_d)}+1}{e^{b(x_d)}+e^{-s_d(x)}}\\
    &\to \frac{1}{e^{-s_d(x_{\leq d-1}',r_0')}}=e^{s_d(x_{\leq d-1}',r_0')} \quad (x \to (x_{\leq {d-1}}', r_0'))
\end{align*}
Similarly, from $\lim_{x\to r_1'}b(x)=\infty$, we have 
\begin{align*}
    \frac{\sigma(s_d(x)+b(x_d))-1}{\sigma(b(x_d))-1}
    &=e^{-s_d(x)}\frac{1+e^{-b(x_d)}}{1+e^{-s_d(x)-b(x_d)}}\\
    &\to e^{-s_d(x_{\leq d-1}, r_1')} \quad (x\to (x_{\leq d-1}',r_1')).
\end{align*}
Therefore, our subclaim has been proved. 
By using L'H\^opital's rule, we have 
\begin{align*}
\lim_{h\to +0}\frac{\sigma(b(r_0'+h))}{h}=\lim_{x\to r_0'}\frac{d(\sigma\circ b)}{dx} (x),\quad 
\lim_{x\to r_1'}\frac{\sigma(b(r_1'+h))-1}{h}=\lim_{x\to r_1'}\frac{d(\sigma\circ b)}{dx}(x). 
\end{align*}
Then, from Subclaim, we obtain 
\begin{align*}
   \frac{\partial S_d}{\partial x_d}(x_{\leq d-1}, r_0')
   &=\lim_{h\to +0} \frac{\sigma(s_d(x_{\leq d-1},r_0'+h)+b(r_0'+h))-0}{h}\\
   &=\lim_{h\to +0} \frac{\sigma(s_d(x_{\leq d-1}, r_0'+h)+b(r_0'+h))}{\sigma(b(r_0+h))}\cdot \frac{\sigma(b(r_0'+h))}{h}\\
   &=e^{s_d(x_{\leq d-1}, r_0')}\lim_{x\to r_0'+0}\frac{d(\sigma\circ b)}{dx}(x), \\
   \frac{\partial S_d}{\partial x_d}(x_{\leq d-1}, r_1')
&=\lim_{h\to -0}\frac{\sigma(s_d(x_{\leq d-1},r_1'+h)+b(r_1'+h))-1}{h}\\
&=\lim_{h\to -0}\frac{\sigma(s_d(x_{\leq d-1},r_1'+h)+b(r_1'+h))-1}{\sigma(b(r_1'+h))-1}\cdot \frac{\sigma(b(r_1'+h))-1}{h}\\
&=e^{s_d(x_{\leq d-1},r_1')}\lim_{x\to r_1'}\frac{d(\sigma \circ b)}{dx}(x).  
\end{align*}
Therefore, Claim 2 was proved. 

Finally, we verify $\frac{\partial S_d}{\partial x_d}(x)$ is continuous on $[r_0',r_1']^d$. 
Fix $x'_{\leq d-1}\in [r_0', r_1']^{d-1}$. 
Since we have 
$\lim_{x\to (x'_{\leq d-1}, r_0')}\frac{d\sigma}{dx}(\sigma_d(x)+b(x_d))\frac{\partial s_d}{\partial x_d}(x)=0$, 
from Claim 2, it is enough to show the following:\\
\textbf{Claim 3.}
\begin{align*}
    \lim_{x\to (x_{\leq d-1}', r_0')}\frac{d \sigma}{d x}(s_d(x)+b(x_d))\frac{db}{d x}(x_d)
    &=e^{s_d(x_{\leq d-1},r_0')}\lim_{x\to r_0'+0}\frac{d (\sigma\circ b)}{dx}(x), \\
    \lim_{x\to (x_{\leq d-1}', r_1')}\frac{d \sigma}{d  x}(s_d(x)+b(x_d))\frac{db}{dx}(x_d)
    &=e^{-s_d(x_{\leq d-1}, r_1')}\lim_{x\to r_1'-0}\frac{d (\sigma\circ b)}{dx}(x). 
\end{align*}
We verify Claim~3. 
We have
\begin{align*}
\frac{d \sigma}{d x}(s_d(x)+b(x_d))\frac{db}{d x}(x_d)
&=\frac{\frac{d\sigma}{dx}(s_d(x)+b(x_d))}{\frac{d\sigma}{dx}(b(x_d))}\frac{d\sigma}{dx}(b(x_d))\frac{db}{dx}(x_d)\\
&=\frac{\frac{d\sigma}{dx}(s_d(x)+b(x_d))}{\frac{d\sigma}{dx}(b(x_d))}\frac{d(\sigma \circ b)}{dx}(x_d). 
\end{align*}
Since we have $\frac{d\sigma}{dx}(x)= \sigma(x)(1-\sigma(x))$, from Subclaim above, Claim 3 follows from 
\begin{align*}
\frac{\frac{d\sigma}{dx}(s_d(x)+b(x_d))}{\frac{d\sigma}{dx}(b(x_d))}
&=\frac{\sigma(s_d(x)+b(x_d))}{\sigma (b(x_d))}\cdot \frac{1-\sigma(s_d(x)+b(x_d))}{1-\sigma(b(x_d))}\\
&\to \begin{cases} 
e^{s_d(x_{\leq d-1}', r_0')} & (x\to (x_{\leq d-1}', r_0'))\\
e^{-s_d(x_{\leq d-1}', r_1')} & (x\to (x_{\leq d-1}', r_1'))
\end{cases}. 
\end{align*}
Therefore, we proved the continuity of $\frac{\partial S_d}{\partial x_d}(x)$. 
\end{proof}
 
\newcommand{\SoSTransformer}[2]{\mathfrak{B}_{#1}(#2)}
\subsection{Sum-of-squares polynomial flows (SoS flows)}
In this section, we prove that \emph{sum-of-squares polynomial flows} \cite{DBLP:conf/icml/JainiSY19} yield \ARFINNs{} with the $\sup$-universal approximation property for $\ConeOneDimTriangular$ (hence for $\CinftyOneDimTriangular$).
Even though \citet{DBLP:conf/icml/JainiSY19} claimed the distributional universality of the SoS flows by providing a proof sketch based on the univariate Stone-Weierstrass approximation theorem, we regard the sketch to be invalid or at least incomplete as it does not discuss the smoothness of the coefficients, i.e., whether the polynomial coefficients can be realized by continuous functions. Here, we provide complete proof that takes an alternative route to prove the $\sup$-universality of the SoS flows via the multivariate Stone-Weierstrass approximation theorem.

\begin{definition}
A \emph{sum-of-squares polynomial flow} (SoS flow) \citep[Equation~(9)]{DBLP:conf/icml/JainiSY19} is a flow layer \(g = (g_1, \ldots, g_d) \colon \ReD \to \ReD\) of the following form:
\begin{equation*}\begin{aligned}
g_k(\x) &:= \SoSTransformer{2r+1}{x_k; C_k(\upto{k-1}{\x})}, \\
\SoSTransformer{2r+1}{z; (c, \ba)} &:= c + \int_0^z \sum_{b=1}^B\left(\sum_{l=0}^r a_{l, b} u^l\right)^2 du,
\end{aligned}\end{equation*}
where \(C_k \colon \Re^{k-1} \to \Re^{B (r+1) + 1}\) is a neural network, \(r \in \Na \cup \{0\}\), and \(B \in \Na\).
We define \(\FSoS\) to be the set of all possible SoS flows.
\end{definition}
\begin{proposition}
The elements of \(\FSoS\) are locally bounded, and \(\INN{\FSoS}\) is a \(\sup\)-universal approximator for \(\ConeOneDimTriangular\).
\end{proposition}
\begin{proof}
The elements of \(\FSoS\) are continuous, hence locally bounded.
The \(\sup\)-universality follows from the Stone-Weierstrass approximation theorem as in the below.
Let \(s = (s_1, \ldots, s_d) \in \ConeOneDimTriangular\), a compact subset \(K \subset \ReD\), and \(\epsilon > 0\) be given.
Then, there exists \(R > 0\) such that \(K \subset [-R, R]^d\).
Since \(s_d(\x)\) is strictly increasing with respect to \(x_d\) and \(s\) is \(C^1\), we have \(\eta(\x) := \frac{\partial s_d}{\partial x_d}(\x) > 0\) and \(\eta\) is continuous.
Therefore, we can apply the Stone-Weierstrass approximation theorem \citep[Corollary~4.50]{FollandReal1999} to \(\sqrt{\eta(\x)}\):
for any \(\delta > 0\), there exists a polynomial \(\pi(x_1, \ldots, x_d)\) such that \(\supRangenorm{[-R, R]^d}{\sqrt{\eta}- \pi} < \delta\).
Then, by rearranging the terms, there exist \(r \in \Na\) and polynomials \(\xi_{l}(x_1, \ldots, x_{d-1})\) such that \(\pi(x_1, \ldots, x_d) = \sum_{l=0}^r \xi_l(x_1, \ldots, x_{d-1})x_d^l\).
Now, define
\begin{equation*}\begin{aligned}
\tilde g_d(\x) &:= s_d(\upto{d-1}{\x}, 0) + \int_0^{x_d} (\pi(\upto{d-1}{\x}, u))^2 du \\
&= s_d(\upto{d-1}{\x}, 0) + \int_0^{x_d} \left(\sum_{l=0}^r \xi_l(x_1, \ldots, x_{d-1})u^l\right)^2 du
\end{aligned}\end{equation*}
and \(\tilde g(\x) := (x_1, \ldots, x_{d-1}, \tilde g_d(\x))\).
Then,
\begin{equation*}\begin{aligned}
\supKnorm{s - \tilde g} &= \sup_{\x \in K} \left|s_d(\x) - \tilde g_d(\x)\right| \\
&= \sup_{\x \in K} \left|s_d(\upto{d-1}{\x}, 0) + \int_0^{x_d} \eta(\upto{d-1}{\x}, u) du - \tilde g_d(\x) \right| \\
&= \sup_{\x \in K} \left|\int_0^{x_d} (\sqrt{\eta(\upto{d-1}{\x}, u)}^2 - \pi(\upto{d-1}{\x}, u)^2) du \right| \\
&\leq R \cdot \sup_{\x \in [-R, R]^d} \left|\sqrt{\eta(\x)}^2 - \pi(\x)^2\right| \\
&= R \cdot \sup_{\x \in [-R, R]^d} |\sqrt{\eta(\x)} + \pi(\x)| \cdot |\sqrt{\eta(\x)} - \pi(\x)| \\
&\leq R \left(\sup_{\x \in [-R, R]^d} 2\sqrt{\eta(\x)} + \delta\right) \delta,
\end{aligned}\end{equation*}
where we used
\begin{equation*}\begin{aligned}
\sup_{\x \in [-R, R]^d}|\sqrt{\eta(\x)} + \pi(\x)| &\leq \sup_{\x \in [-R, R]^d}|2\sqrt{\eta(\x)}| + |\sqrt{\eta(\x)} - \pi(\x)| \\
&\leq \sup_{\x \in [-R, R]^d} 2\sqrt{\eta(\x)} + \delta.
\end{aligned}\end{equation*}
It is straightforward to show that there exists \(g \in \FSoS\) such that \(\supKnorm{\tilde g - g} < \frac{\epsilon}{2}\) by approximating each of \(s_d(\upto{d-1}{\x})\) and \(\xi_l\) on \(K\) using neural networks.
Finally, taking \(\delta\) to be small enough so that \(\supKnorm{s - \tilde g} < \frac{\epsilon}{2}\) holds, the assertion is proved.
\end{proof}

\section{Using permutation matrices instead of \(\FLin\) in the definition of \(\INN{\ARFINNFlow}\)}
\label{sec:appendix:elementary matrix}
In terms of representation power, there is no essential difference between using the permutation group and using the general linear group in Definition~\ref{def: INNM}.
In fact, one can express the elementary operation matrices (hence the regular matrices) by combining affine coupling flows, permutations.

From this result, we can see that employing \(\FLin\) in Definition~\ref{def: INNM} instead of the permutation matrices is not an essential requirement for the universal approximation properties to hold.
For this reason, we believe that the empirically reported difference in the performances of Glow \cite{KingmaGlow2018} and RealNVP \cite{DinhDensity2016a} is mainly in the efficiency of approximation rather than the capability of approximation. 
\begin{lemma}
\label{lem:appendix: sign flip and permutations}
We have 
\begin{align}
    \INN{\FSACFH}=\{ W_1\circ g_1 \circ \cdots \circ W_n \circ g_n ~:~ g_i\in\FSACFH, W_i\in \mathfrak{S}_d\},\label{INNHACF=HACF + permutation}
\end{align}
where $\mathfrak{S}_d$ is the permutation group of degree $d$.
\end{lemma}
\begin{proof}
Since any translation operator (i.e., addition of a constant vector) can be easily represented by the elements of $\FSACFH$ and permutations, it is enough to show that any element of ${\rm GL}(n,\R)$ can be realized by a finite composition of elements of $\FSACFH$ and $\mathfrak{S}_d$.
To show that, it is sufficient to consider only the elementary matrices. 
Row switching comes from $\mathfrak{S}_d$.
Moreover, element-wise sign flipping can be described by a composition of finite elements of $\FSACFH$. To see this, first observe that
\begin{align}
\left(
\begin{array}{cc}
-1&0\\
0&1
\end{array}
\right)
=
\left(
\begin{array}{cc}
1&0\\
1&1
\end{array}
\right)
\left(
\begin{array}{cc}
0&1\\
1&0
\end{array}
\right)
\left(
\begin{array}{cc}
1&0\\
-1&1
\end{array}
\right)
\left(
\begin{array}{cc}
0&1\\
1&0
\end{array}
\right)
\left(
\begin{array}{cc}
1&0\\
1&1
\end{array}
\right)
\left(
\begin{array}{cc}
0&1\\
1&0
\end{array}
\right)
\nonumber
\end{align}
holds.
Now, any lower triangular matrix with positive diagonals can be described by a composition of finite elements of $\FSACFH$.
Therefore, any diagonal matrix whose components are $\pm1$ can be described by a composition of elements in $\FSACFH$ and $\mathfrak{S}_d$.
Therefore, any affine transform is an element of the right hand side of \eqref{INNHACF=HACF + permutation}.
\end{proof}

\section{Other related work}
\label{sec:appendix:other related work}

In this section, we elaborate on the relation of the present paper and the existing literature.

\paragraph{Approach to make universal approximators by augmenting the dimensionality.}
\citet{ZhangApproximation2019a} showed that invertible residual networks (i-ResNets) \citep{pmlr-v97-behrmann19a} and neural ordinary differential equations (NODEs) \citep{ChenNeural2018a,DupontAugmented2019a} can be turned into universal approximators of homeomorphisms by increasing the dimensionality and padding zeros.

Given that, one may wonder if we can apply a similar technique to augment \ARFINN{} to have the universality, which can bypass the proof techniques developed in this study.
However, there is a problem that the approach can undermine the exact invertibility of the model: unless the model is ideally trained so that it always outputs zeros in the zero-padded dimensions, the model can no longer represent an invertible map operating on the original dimensionality.
On the other hand, we showed the universality properties of certain \ARFINNs{} without introducing the complication arising from the dimensionality augmentation.

\end{appendices}
\end{document}